\documentclass[twoside,11pt]{article}

\usepackage{jmlr2e}

\jmlrheading{11}{2010}{1-48}{6/10}{-}{B. Nelson, B.~I.~P. Rubinstein, L. Huang, A.~D. Joseph, S.~J. Lee, S.~Rao and J.~D. Tygar} %

\ShortHeadings{Query Strategies for Evading \ConvexClass}{Nelson, Rubinstein, Huang, Joseph, Lee, Rao and Tygar}
\firstpageno{1}

\usepackage{xkeyval}
\usepackage{color}
\usepackage{algorithm}
\usepackage{algorithmic}
\usepackage{MathLib}
\usepackage{xcolor}
\usepackage{amsmath}
\usepackage{multirow}
\usepackage{textcomp}
\usepackage{pstricks,pst-plot,pst-3dplot,pst-infixplot,pst-xkey}
\usepackage{NearOpt-LaTeX-Lib}

\usepackage{paralist}

\newtheorem{program}[theorem]{Program}
\newenvironment{Program}[1]
  {\begin{program} \quad #1 \vspace{0mm} \\ \mbox{} \quad \quad \textbf{s.t.} \begin{minipage}[c]{.8\textwidth} \begin{trivlist}}
  {\end{trivlist} \end{minipage} \end{program}}

\newenvironment{proofof}[1]{\par\noindent{\bf Proof of {#1}\ }}{\hfill\BlackBox\\[2mm]}

\newcommand{\algoName}[1]{\textsc{#1}}
\newtheorem{algo}{Algorithm}
\newenvironment{Algorithm}[2]
{\begin{algo} \label{#2} \quad \algoName{#1}  \\ \rule[1cm]{\linewidth}{.5pt}\vspace{-15mm}
    \begin{algorithmic}\rm}
{\end{algorithmic}\vspace{-3mm}\rule{\linewidth}{.5pt}\end{algo}}

\newcommand{\picwidth}{\linewidth}

\newcounter{subfig}[figure]
\renewcommand{\thesubfig}{(\alph{subfig})}
\newcommand{\subref}[2]{\ref{#1}\ref{#2}}

\newcommand{\ie}{\emph{i.e.},}
\newcommand{\eg}{\emph{e.g.},}
\newcommand{\cf}{\emph{cf.}}
\newcommand{\term}[1]{\emph{#1}}
\newcommand{\nth}[2]{\ensuremath{{#1}^{\mbox{\scriptsize #2}}}} %

\begin{document}

\title{Query Strategies for Evading \ConvexClass}

\author{\name Blaine Nelson \email nelsonb@cs.berkeley.edu \\
	\addr Computer Science Division\\
	University of California\\
	Berkeley, CA 94720-1776, USA
	\AND
	\name Benjamin I.~P. Rubinstein \email benr@cs.berkeley.edu \\
	\addr Microsoft Research\\
	Mountain View, CA 94043, USA
	\AND
	\name Ling Huang \email ling.huang@intel.com \\
	\addr Intel Labs Berkeley \\
	Berkeley, CA 94709, USA
	\AND
	\name Anthony D. Joseph \email adj@cs.berkeley.edu \\
	\addr Computer Science Division\\
	University of California\\
	Berkeley, CA 94720-1776, USA
	\AND
	\name Steven J. Lee \email stevenjlee@berkeley.edu  \\
	\name Satish Rao \email satishr@cs.berkeley.edu \\
	\name J.~D. Tygar \email tygar@cs.berkeley.edu \\
	\addr Computer Science Division \\
	University of California \\
	Berkeley, CA 94720-1776, USA}

\editor{Leslie Pack Kaelbling}

\maketitle

\begin{abstract}%
  Classifiers are often used to detect miscreant activities. We study
  how an adversary can systematically query a classifier to elicit
  information that allows the adversary to evade detection while
  incurring a near-minimal cost of modifying their intended
  malfeasance. We generalize the theory of \citet{lowd05adversarial}
  to the family of \convexClass\ that partition input space into two
  sets one of which is convex. We present query algorithms for this
  family that construct undetected instances of approximately minimal
  cost using only polynomially-many queries in the dimension of the
  space and in the level of approximation. Our results demonstrate
  that near-optimal evasion can be accomplished without
  reverse-engineering the classifier's decision boundary. We also
  consider general \LP\ costs and show that near-optimal evasion on
  the family of \convexClass\ is generally efficient for both positive
  and negative convexity for all levels of approximation if $p=1$.

\end{abstract}

\begin{keywords}
Query Algorithms, Evasion, Reverse Engineering, Adversarial Learning
\end{keywords}

\section{Introduction}

A number of systems and security engineers have proposed the use of
machine learning techniques 
to filter or detect miscreant activities in a variety of applications;
\eg\ spam, intrusion, virus, and fraud detection. All known detection techniques have blind spots: classes of
miscreant activity that fail to be detected.
While learning algorithms allow the detection algorithm to adapt over
time, real-world constraints on the learner typically allow an
adversary to programmatically find vulnerabilities. We consider how an
adversary can
systematically discover blind spots by querying a fixed or learning-based
detector to find a low cost (for some cost function) instance that the
detector does not filter. As a motivating example, consider a spammer who
wishes to
minimally modify a spam message so it is not classified as a spam
(here cost is a measure of how much the spam must be modified). 
There are a variety of domain specific mechanisms an adversary can use to
  observe the classifier's response to a query; \eg\ the spam filter
  of a public email system can be observed by creating a dummy account
  on that system and sending the queries to that account. We assume the
  attacker has access to a membership oracle for the filter.
By observing the responses of the spam detector, the spammer can
search for a modification while using as few queries as possible.

The problem of near-optimal evasion (\ie\ finding a low cost negative
instance with few queries) was first posed by
 \citet{lowd05adversarial}. We continue their investigation by
generalizing their results to the family of \convexClass---classifiers that
partition their instance space into two sets one of which is convex.
The family of \convexClass\ is a particularly important and natural
class to examine, as it includes the family of linear classifiers
studied by Lowd and Meek as well as anomaly detection classifiers
using bounded PCA~\citep{LCD04}, anomaly detection algorithms that
use hyper-sphere boundaries~\citep{PRML}, one-class classifiers that
predict anomalies by thresholding the log-likelihood of a log-concave (or 
uni-modal) density function, and quadratic classifiers of the form
$\dpt^\top \mat{A} \dpt + \vec{b}^\top \dpt + c \ge 0$ if $\mat{A}$ is
semidefinite, to name a few.  Furthermore, the family of \convexClass\ also
includes more complicated bodies such as the countable intersection of
halfspaces, cones, or balls.
We also show that near-optimal evasion does not require reverse
engineering the classifier's decision boundary, which is the approach taken by
\citet{lowd05adversarial} for evading linear classifiers.
Our algorithms for evading
\convexClass\ do not require fully estimating the classifier's
boundary (which is hard in the general convex case; see
\citealp{LearningConvexIsHard}) or otherwise reverse-engineering the
classifier's state. Instead, we directly search for a minimal-cost
evading instance.  Our algorithms require only polynomial-many
queries, with one algorithm solving the linear case with better query
complexity than the previously-published reverse-engineering
technique.

This paper is organized as follows. We overview past work related to
near-optimal evasion in the remainder of this section.
In Section~\ref{sec:problem} we formalize the near-optimal evasion problem, and
review Lowd and Meek's definitions and results. We present algorithms for 
evasion that are near-optimal under \LP[1] cost in
Section~\ref{sec:conv-evasion} and we consider minimizing general \LP[p] costs
in Section~\ref{sec:LPCosts}. We conclude the paper by discussing future
directions for near-optimal evasion of classifiers in
Section~\ref{sec:conclusion}.

\subsection{Related Work}

\citet{lowd05adversarial} first explored near-optimal evasion, and developed a
method that reverse-engineered linear classifiers. Our approach generalizes
their result and improves upon it in three significant ways.
\begin{itemize}
\item We consider a more general family of classifiers: the family of
  \convexClass\ that partition the space of instances into two sets
  one of which is convex.  This family subsumes the family of linear
  classifiers considered by Lowd and Meek.
\item Our approach does not fully estimate the classifier's decision boundary
  (which is generally hard; see \citealt{LearningConvexIsHard}) or
  reverse-engineer the classifier's state; instead, we directly search
  for an instance that the classifier recognizes as negative that is
  close to the desired attack instance (an evading instance of
  near-minimal cost).
\item Even though our algorithms find solutions for a more general
  family of classifiers, our algorithms still only use a limited
  number of queries: they require only a number of queries polynomial
  in the dimension of the instance space. Moreover, our \KMLS\
  (Algorithm~\ref{alg:kmls}) solves the linear case with fewer queries
  than the previously-published reverse-engineering technique.
\end{itemize}

\citet{dalvi04adversarial} use a cost-sensitive game theoretic
approach to preemptively patch a classifier's blind
spots~\citep{dalvi04adversarial}. They construct a modified classifier
designed to detect optimally modified instances. This work is
complementary to our own; we examine optimal evasion strategies while
they have studied mechanisms for adapting the classifier.  In this
paper we assume the classifier is not adapting during evasion.

A number of authors have studied evading sequence-based intrusion
detector systems
(IDSs)~\citep{tan-etal-2002-undermining,wagner-soto-2002-mimicry}.  In
exploring \emph{mimicry attacks} these authors demonstrated that real
IDSs can be fooled by modifying exploits to mimic normal behaviors.
These authors used offline analysis of the IDSs to construct their
modifications; by contrast, our modifications are optimized by
querying the classifier.

The field of active learning also studies
a form of query-based optimization~\citep{ActLearning}.
While active learning and near-optimal evasion are similar in their exploration
of querying strategies, the objectives for these two settings are quite
different (see Section~\ref{sec:opt-evasion}).

\section{Problem Setup}\label{sec:problem}

We begin by introducing our notation and our assumptions.  First, we
assume that instances are represented in a feature space $\xspace$
which is $\dims$-dimensional Euclidean space\footnote{Lowd and Meek
  also consider integer and Boolean-valued instance spaces and derive
  results for several classes of Boolean-valued
  learners.} $\xspace = \reals^\dims$.  Each component of an
instance $\dpt \in \xspace$ is a \term{feature} which we denote as
$\dptc_d$.
We denote each coordinate vector of the form $(0,\ldots,1,\ldots,0)$
with a $1$ only at the \nth{d}{th} feature as $\coordinVect{d}$. We
assume that the feature space representation is known to the adversary
and there are no restrictions on the adversary's queries; \ie\ any
point in feature space $\xspace$ can be queried by the adversary.
These assumptions may not be true in every real-world setting, but they allow us
to investigate strategies taken by a worst-case adversary. We revisit this
assumption in Section~\ref{sec:conclusion}.

We further assume the target classifier $\classifier$ belongs to a
family of classifiers $\classSpace$. Any classifier $\classifier \in
\classSpace$ is a mapping from feature space $\xspace$ to its response
space $\yspace$; \ie\ $\classifier: \xspace \to \yspace$. We assume
the adversary's attack will be against a fixed $\classifier$ so the
learning method and the training data used to select $\classifier$ are
irrelevant. We assume the adversary does not know $\classifier$ but
knows its family $\classSpace$. We also restrict our attention to
binary classifiers and use $\yspace = \set{\negLbl,\posLbl}$.

We assume $\classifier \in \classSpace$ is deterministic and so it
partitions $\xspace$ into 2 sets---the positive class $\xplus =
\set[\prediction{\dpt} = \posLbl]{\dpt \in \xspace}$ and the negative
class $\xminus = \set[\prediction{\dpt} = \negLbl]{\dpt \in \xspace}$.
We take the negative set to be \emph{normal} instances.
We assume that the adversary is aware of at least one
instance in each class, $\dpt^- \in \xminus$ and $\xtarget \in
\xplus$, and can observe $\prediction{\dpt}$ for any $\dpt$ by issuing
a \term{membership query} (see Section~\ref{sec:discuss} for a more detailed
discussion of this assumption).

\subsection{Adversarial Cost}
\label{sec:adCost}

We assume the adversary has a notion of utility over the instance
space which we quantify with a cost function $\adCost : \xspace
\to \realnn$; \eg\ for a spammer this could be edit distance on email
messages. The adversary wishes to optimize $\adCost$ over the
negative class, $\xminus$; \eg\ the spammer wants to send spam that
will be classified as normal email (\negLbl) rather than as spam
(\posLbl). We assume this cost function is a distance to some instance
$\xtarget \in \xplus$ that is most desirable to the adversary.
We focus on the general class of weighted \LP[p] ($0 < p \le \infty$)
cost functions:
\begin{equation}
  \lpCostFunc[\featCost]{p}{\dpt}
  = \left(\sum_{d=1}^{\dims}{\ithCost{d} \left|\dptc_d - \xtargetc_d\right|^p}\right)^{1/p} \enspace,
  \label{eq:weightedL1}
\end{equation}
where $ 0 < \ithCost{d} < \infty$ is the relative cost the adversary
associates with the \nth{d}{th} feature. We also consider the
cases when some features have $\ithCost{d}=0$ (adversary
doesn't care about the \nth{d}{th} feature) or $\ithCost{d}=\infty$
(adversary requires the \nth{d}{th} feature to match $\dptc^A_d$).
Weighted \LP[1] costs are particularly appropriate
for many adversarial problems since costs are assessed based on the
degree to which a feature is altered and the adversary typically is
interested in some features more than others. Unless stated otherwise, we take
``\LP[1] cost'' to mean a weighted \LP[1] cost in the sequel. The $\LP[1]$-norm
is a natural measure of edit distance for email spam, while larger weights
can model tokens that are more costly to remove (\eg\ a payload URL).
As with Lowd and Meek, we focus primarily on \LP[1] costs in
Section~\ref{sec:conv-evasion} before exploring general \LP[p] costs in
Section~\ref{sec:LPCosts}.  We use
$\ball{C}{\adCost} = \set[\adCostFunc{\dpt} \le C]{\dpt \in \xspace}$
to denote the cost-ball (or sublevel set) with cost no more than $C$.
For instance, $\ball{C}{\lpCost{1}}$ is the set of instances that do
not exceed an \LP[1] cost of $C$ from the target $\xtarget$.
\citet{lowd05adversarial} define \term{minimal adversarial cost (\MAC)} of a
classifier $\classifier$ to be the value
\[
  \function{\MAC}{\classifier,\adCost}
  \defAs \inf_{\dpt \in \xminus[\classifier]}\left[ \adCostFunc{\dpt}
  \right] \enspace;
\]
\ie\ the greatest lower bound on the cost obtained by any negative
instance.  They further define a data point to be an
$\multGoal$-approximate \emph{instance of minimal adversarial cost
  (\kIMAC)} if it is a negative instance with a cost no more than a
factor $(1+\multGoal)$ of the \MAC; \ie\ every \kIMAC\ is a member of
the set\footnote{We use `\kIMAC' to refer both to this set and its
  members. The meaning will be clear from the context.}
\begin{equation}
  \label{eq:kIMAC}
  \function{\kIMAC}{\classifier,\adCost} \defAs
  \set[\adCostFunc{\dpt} \le 
  (1+\multGoal)\cdot\function{\mathrm{MAC}}{\classifier,\adCost}]
  {\dpt \in \xminus[\classifier]}
  \enspace .
\end{equation}
The adversary's goal is to find an \kIMAC\
efficiently, while issuing as few queries as possible.

\subsection{Search Terminology}
\label{sec:binSearch}

The notion of near-optimality introduced in Eq.~\eqref{eq:kIMAC} is
that of \term{multiplicative optimality}; \ie\ an \kIMAC\ must have
a cost within a factor of $(1+\multGoal)$ of the \MAC. However, the
results of this paper can also be immediately adopted for
\term{additive optimality} in which we seek instances with cost no more than
$\addGoal>0$ \emph{greater} than the \MAC.  To differentiate between
these notions of optimality, we will use the notation
$\kIMAC[\multGoal]^{(\ast)}$ to refer to the set in
Eq.~\eqref{eq:kIMAC} and define an analagous set
$\kIMAC[\addGoal]^{(+)}$ for addative optimality as
\begin{equation}
  \label{eq:additiveIMAC}
  \function{\kIMAC[\addGoal]^{(+)}}{\classifier,\adCost} \defAs
  \set[\adCostFunc{\dpt} \le 
  \addGoal + \function{\mathrm{MAC}}{\classifier,\adCost}]
  {\dpt \in \xminus[\classifier]}
  \enspace .
\end{equation}
We use the terms $\kIMAC[\multGoal]^{(\ast)}$ and
$\kIMAC[\addGoal]^{(+)}$ to refer both to the sets defined in
Eq.~\eqref{eq:kIMAC} and~\eqref{eq:additiveIMAC} as well as the
members of them---the usage will be clear from the context.

Either notion of optimality allows us to efficiently use bounds on the
\MAC\ to find an $\kIMAC[\multGoal]^{(\ast)}$ or an
$\kIMAC[\addGoal]^{(+)}$. Suppose there is a negative instance,
$\dpt$, with cost \maxCost[] and all instances with cost no more than
\minCost[] are positive; \ie\ \maxCost[] is an upper bound and
\minCost[] is a lower bound on the \MAC: $\minCost[] \le
\function{\MAC}{\classifier,\adCost} \le \maxCost[]$. Then the
negative instance $\dpt$ is \multGoal-multiplicatively optimal if
$\maxCost/\minCost \le (1+\multGoal)$ whereas it is
\addGoal-additively optimal if $\maxCost - \minCost \le \addGoal$. In
the sequel, we will consider algorithms that can achieve either
additive or multiplicative optimality. These algorithms employ binary
search strategies to iteratively reduce the gap between any \maxCost[]
and \minCost[]. Namely, if we can determine whether an intermediate
cost establishes a new upper or lower bound on \MAC, then our binary
search strategies can iteratively reduce the \nth{t}{th} gap between
$\maxCost[t]$ and $\minCost[t]$. We now provide common terminology for
the binary search and in Section~\ref{sec:conv-evasion} we use
convexity to establish a new bound at each iteration.

\begin{lemma}
  If an algorithm can provide bounds $\minCost[] \le
  \function{\MAC}{\classifier,\adCost} \le \maxCost[]$, then this
  algorithm has achieved (1) $(\maxCost[]-\minCost[])$-additive
  optimality and (2) $(\frac{\maxCost[]}{\minCost[]} - 1)$-multiplicative
  optimality.
\end{lemma}

In the \nth{t}{th} iteration of an additive binary search, the
\term{additive gap} between the \nth{t}{th} bounds is given by
$\gap{t}^{(+)} = \maxCost[t] - \minCost[t]$ with $\gap{0}^{(+)}$
defined accordingly by the initial bounds \maxCost\ and \minCost.  The
search uses a proposal step of $\proposal{t} =(\maxCost[t] +
  \minCost[t])/2$, a stopping criterion of $\gap{t}^{(+)} \le
\addGoal$ and achieves \addGoal-additive optimality in
\begin{equation}
  \addSteps = \left\lceil\log_2\left[\frac{\gap{0}^{(+)}}{\addGoal}\right]
  \right\rceil
  \label{eq:Ladd}
\end{equation}
steps. Binary search has the best worst-case query complexity for
achieving \addGoal-additive optimality.

Binary search can also be used for multiplicative optimality by
searching in exponential space. By rewriting our upper and lower
bounds as $\maxCost[]=2^a$ and $\minCost[]=2^b$, the multiplicative
optimality condition becomes $a-b \le \log_2(1+\multGoal)$, an
additive optimality condition. Thus, binary search on the exponent
achieves \multGoal-multiplicative optimality and does so with the
fewest queries. The \term{multiplicative gap} of the \nth{t}{th}
iteration is $\gap{t}^{(\ast)} = \maxCost[t] / \minCost[t]$ with
$\gap{0}^{(\ast)}$ defined accordingly by the initial bounds \maxCost\
and \minCost.  The \nth{t}{th} query is $\proposal{t} =
\sqrt{\maxCost[t]\cdot\minCost[t]}$, the stopping criterion is
$\gap{t}^{(\ast)} \le 1+\multGoal$ and achieves
\multGoal-multiplicative optimality in
\begin{equation}
  \multSteps = \left\lceil\log_2\left[\frac{\log_2\left(\gap{0}^{(\ast)}\right)}{\log_2(1+\multGoal)}\right]\right\rceil
  \label{eq:Lmult}
\end{equation}
steps. Multiplicative optimality only makes sense when both \maxCost\
and \minCost\ are strictly positive.

Binary searches for additive and multiplicative optimality differ in
their proposal step and their stopping criterion. For additive
optimality, the proposal is the arithmetic mean $\proposal{t} =
(\maxCost[t]+\minCost[t])/2$ and search stops when
$\gap{t}^{(+)} \le \addGoal$, whereas for multiplicative optimality,
the proposal is the geometric mean $\proposal{t} =
\sqrt{\maxCost[t]\cdot\minCost[t]}$ and search stops when
$\gap{t}^{(\ast)} \le 1+\multGoal$.  For the remainder of this paper,
we will address \multGoal-multiplicative optimality for an \kIMAC\
(except where explicitly noted) and define $\steps =
\multSteps$ and $\gap{t} = \gap{t}^{(\ast)}$.  Nonetheless,
our algorithms are immediately adapted to additive optimality by
simply changing the proposal step, stopping condition, and the
definitions of $\multSteps$ and \gap{t}. 

\subsection{Near-Optimal Evasion}
\label{sec:opt-evasion}

\citet{lowd05adversarial} introduce the concept of \term{adversarial classifier
  reverse engineering (ACRE) learnability} to quantify the difficulty
of finding an \kIMAC\ instance for a particular family of classifiers
\classSpace, and a family of adversarial costs
\adCostSpace. Using our notation, their
definition of \kACRE\ is
\begin{quote}
  A set of classifiers \classSpace\ is \kACRE\ under a set of cost
  functions \adCostSpace\ if an algorithm exists such that for all
  $\classifier \in \classSpace$ and $\adCost \in \adCostSpace$, it can
  find a $\dpt \in \function{\kIMAC}{\classifier,\adCost}$ using only
  polynomially many membership queries in $\dims$, the encoded size of
  $\classifier$, and the encoded size of $\dpt^+$ and $\dpt^-$.
\end{quote}
In generalizing their result, we slightly alter their definition of
query complexity. First, to quantify query complexity we only use the
dimension $\dims$ and the number of steps \multSteps\ required by a
univariate binary search to narrow the gap between initial bounds
\minCost\ and \maxCost\ to less than $(1+\multGoal)$.\footnote{Using
  the encoded sizes of $\classifier$, $\dpt^+$, and $\dpt^-$ in
  defining \kSearchable\ is problematic. For our purposes, it is clear
  that the encoded size of both $\dpt^+$ and $\dpt^-$ is $\dims$ so it
  is unnecessary to include additional terms for their size. Further
  we allow for families of non-parametric classifiers for which the
  notion of \emph{encoding size} is ill-defined but is also
  unnecessary for the algorithms we present.  In extending beyond
  linear and parametric family of classifiers, it is not
  straightforward to define the encoding size of our classifier
  $\classifier$. One could use notions such as the \term{VC-dimension}
  of $\classSpace$ or its \term{covering number}~\citep{BartlettBook}
  but it is unclear why size of the classifier is important in
  quantifying the complexity of \kIMAC\ search. Moreover, as we
  demonstrate in this paper, there are non-parametric families of
  classifiers for which \kIMAC\ search is polynomial in $\dims$
  alone.}  Second, we assume the adversary only has two initial points
$\dpt^- \in \xminus$ and $\xtarget \in \xplus$ (the original setting
required a third $\dpt^+ \in \xplus$): we restrict our setting to the
case of $\dpt^A \in \xplus$, yielding simpler search
procedures.\footnote{However, as is apparent in the algorithms we
  demonstrate, using $\dpt^+ = \dpt^A$ makes the attacker less covert
  since it is significantly easier to infer the attacker's intentions
  based on their queries.  (Covertness is not an explicit goal in
  \kIMAC\ search but it would be a requirement of many real-world
  attackers.) However, since our goal is not to design real attacks
  but rather analyze the best possible attack so as to understand our
  classifier's vulnerabilities, covertness can be ignored.}  Finally,
our algorithms do not reverse engineer the decision boundary, so
``ACRE'' would be a misnomer here.  Instead we refer to the overall
problem as \term{Near-Optimal Evasion} and replace \kACRE\ with the
following definition of \kSearchable.
\begin{quote}
  A family of classifiers \classSpace\ is \term{\kSearchable} under a
  family of cost functions \adCostSpace\ if for all $\classifier \in
  \classSpace$ and $\adCost \in \adCostSpace$, there is an algorithm
  that finds $\dpt \in \function{\kIMAC}{\classifier,\adCost}$ using
  polynomially many membership queries in $\dims$ and \steps. We will
  refer to such an algorithm as \term{efficient}.
\end{quote}

Unlike Lowd and Meek's approach, our algorithms construct queries to
provably find an \kIMAC\ without reverse engineering the classifier's
decision boundary.
Efficient query-based reverse engineering for
$\classifier \in \classSpace$ is sufficient for minimizing
$\adCost$ over the estimated negative space.  However, generally
reverse engineering (\term{active learning}) is an expensive approach
for near-optimal evasion, requiring query complexity that is exponential in
the feature space dimension for general convex
classes~\citep{LearningConvexIsHard},
while finding an \kIMAC\ need not be---the requirements for finding an
\kIMAC\ differ significantly from the objectives of reverse
engineering approaches such as active learning.  Both approaches use queries to
reduce the size of version space $\hat{\classSpace} \subset
\classSpace$, the set of classifiers consistent with the adversary's
membership queries. However reverse engineering approaches minimize
the expected number of disagreements between members of
$\hat{\classSpace}$.
In contrast, to find an \kIMAC, we only need to provide a single
instance $\dpt^\dagger \in \function{\kIMAC}{\classifier,\adCost}$ for
all $\classifier \in \hat{\classSpace}$, while leaving the classifier
largely unspecified; \ie
\[
\bigcap_{\classifier \in \hat{\classSpace}} \function{\kIMAC}{\classifier,\adCost} \neq \emptyset \enspace.
\]
This objective allows the classifier to be unspecified in much of
$\xspace$. We present algorithms for \kIMAC\ search on a family of
classifiers that generally cannot be efficiently reverse
engineered---the queries we construct necessarily elicit an \kIMAC\
only; the classifier itself will be underspecified in large regions of
$\xspace$ so our techniques do not reverse engineer the classifier.

\subsection{Multiplicative vs. Additive Optimality}

Additive and multiplicative optimality are intrinsically related by
the fact that the optimality condition for multiplicative optimality
$\maxCost[t] / \minCost[t] \le 1+\multGoal$ can be rewritten as
additive optimality condition $\log_2\maxCost[t] - \log_2\minCost[t]
\le \log_2(1+\multGoal)$. From this equilence we can take $\addGoal =
\log_2(1+\multGoal)$ and use the additive optimality criterion on the
logarithm of the cost. However, this equivalence also leads to two
differnces between these notions of optimality.

First, multiplicative optimality only makes sense when 
\minCost\ is strictly positive (we will need this assumption for our
algorithms) whereas additive optimality can still be achieved if
$\minCost = 0$. In this special case, $\xtarget$ is on the boundary of
\xplus\ and there is no $\kIMAC[\multGoal]^{(\ast)}$ for any
$\multGoal > 0$. Practically speaking though, this is a minor
hinderance---as we demonstrate in Section~\ref{sec:specialPosCases},
there is an algorithm that can efficiently establish any lower bound
\minCost\ if such a lower bound exists.

Second, the additive optimality criterion is not \term{scale
  invariant} (\ie\ any instance $\dpt^\dagger$ that satisfies the
optimality criterion for cost \adCost\ also satisfies it for
$\function{\adCost^\prime}{\dpt} = s \cdot \adCostFunc{\dpt}$ for any
$s > 0$) whereas multiplictative optimality is scale
invariant. Additive optimility is, however, \term{shift
invariant} (\ie\ any instance $\dpt^\dagger$ that satisfies the
optimality criterion for cost \adCost\ also satisfies it for
$\function{\adCost^\prime}{\dpt} = s + \adCostFunc{\dpt}$ for any $s \ge 0$)
whereas multiplicative optimality is not. Scale invariance is
typically more salient because if the cost function is also scale
invariant (all proper norms are) then the optimality condition is
invariant to a rescaling of the underlying feature space; \eg\ a
change in units for all features. Thus, multiplicative optimality is
a unitless notion of optimality whereas additive optimality is not.
The following result is a consequence of additive optimality's lack of scale
invariance.

\begin{theorem}
  \label{thm:addNoEff}
  If for some hypothesis space $\classSpace$, cost function \adCost,
  and any initial bounds $0 < \minCost < \maxCost$ on the
  $\function{\MAC}{\classifier,\adCost}$ for some $\classifier
  \in \classSpace$, there exists some $\bar{\multGoal} > 0$ such that
  no efficient query-based algorithm can find an
  $\kIMAC[\multGoal]^{(\ast)}$ for any $0 < \multGoal \le
  \bar{\multGoal}$, then there is no efficient query-based algorithm
  that can find a $\kIMAC[\addGoal]^{(+)}$ for any $0 < \addGoal \le
  \bar{\multGoal} \cdot \maxCost$.
\end{theorem}
\begin{proof}
  We will proceed by contraposition. If there is an efficient
  query-based algorithm that can find a $\dpt \in
  \kIMAC[\addGoal]^{(+)}$ for some $0 < \addGoal \le \bar{\multGoal}
  \cdot \maxCost$, then, by definition of $\kIMAC[\addGoal]^{(+)}$,
  $\adCostFunc{\dpt} \le \addGoal +
  \function{\MAC}{\classifier,\adCost}$. Taking $\addGoal = \multGoal
  \cdot \function{\MAC}{\classifier,\adCost}$ for some $\multGoal >
  0$, we have equivalently achieved $\adCostFunc{\dpt} \le (1 +
  \multGoal) \function{\MAC}{\classifier,\adCost}$; \ie\ $\dpt \in
  \kIMAC[\multGoal]^{(\ast)}$. Moreover, since
  $\function{\MAC}{\classifier,\adCost} \le \maxCost$, this efficient
  algorithm is able to find a $\kIMAC[\multGoal]^{(\ast)}$ for some
  $\multGoal \le \bar{\multGoal}$.
\end{proof}

\begin{corollary}
  If for some hypothesis space $\classSpace$, cost function \adCost,
  there exists some $\bar{\multGoal} > 0$ such that no efficient
  query-based algorithm can find an $\kIMAC[\multGoal]^{(\ast)}$ for
  any $0 < \multGoal \le \bar{\multGoal}$, then there is no efficient
  query-based algorithm that can find a $\kIMAC[\addGoal]^{(+)}$ for
  any $\addGoal$.
\end{corollary}
\begin{proof}
  This follows from Theorem~\ref{thm:addNoEff} since $\maxCost$ may be
  arbitrarily large and $\bar{\multGoal} > 0$.
\end{proof}

This corollary demonstrates that the lack of scale invariance
in the additive optimality condition allows for the feature space to
be arbitrarily rescaled until any fixed level of additive
optimality can no longer be achieved; \ie\ the units of the cost
determine whether a particular level of additive accuracy can be
achieved whereas multiplicative costs are unitless.
\section{Evasion of Convex Classes}
\label{sec:conv-evasion}

\begin{figure}[t]
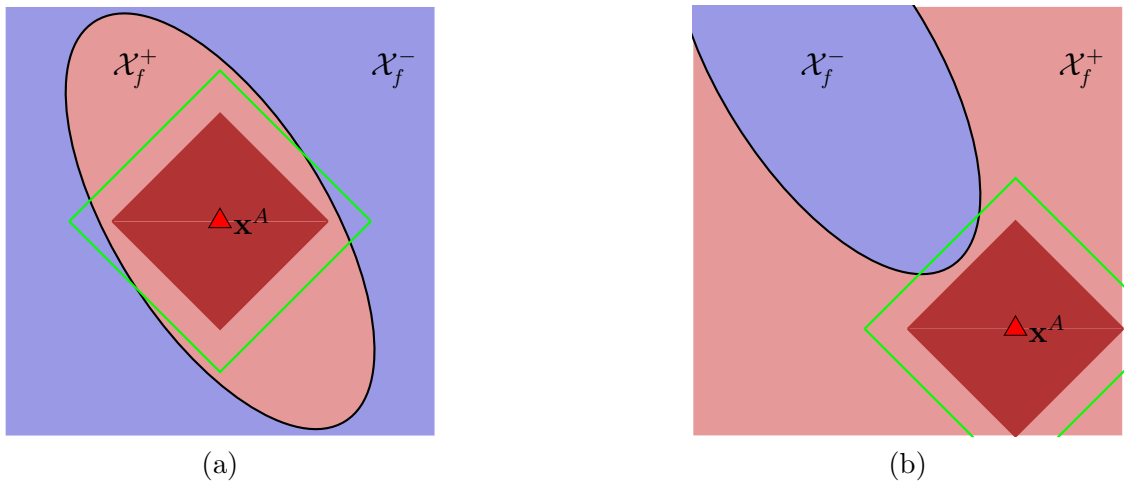

  \begin{minipage}{.4\linewidth}
    \begin{center}
      \renewcommand{\picwidth}{0.47\linewidth}
      \psset{unit=\picwidth}
      \pspicture*(-1,-1)(1,1)
      \psframe[linestyle=none,fillstyle=solid,fillcolor=negColor](-1,-1)(1,1)
      \pscustom[linecolor=boundColor,fillstyle=solid,fillcolor=posColor]{
        \rotate{32}
        \psellipse(0,0)(.5,1.1)
      }
      \lpcontourplot[linecolor=innerBallCol,fillstyle=solid,fillcolor=innerBallCol]{1}{0.5}
      \lpcontourplot[linecolor=green]{1}{0.7}
      \placeTarget
      \rput[lt]{*0}(-0.5,0.8){\large $\xplus$}
      \rput[lt]{*0}(0.7,0.8){\large $\xminus$}
      \endpspicture \\
      \refstepcounter{subfig}
      \thesubfig
      \label{fig:convexPos}
    \end{center}
  \end{minipage} \hfill
  \begin{minipage}{.4\linewidth}
    \begin{center}
      \renewcommand{\picwidth}{0.47\linewidth}
      \psset{unit=\picwidth}
      \pspicture*(-1.5,-.5)(.5,1.5)
      \psframe[linestyle=none,fillstyle=solid,fillcolor=posColor](-1.5,-.5)(.5,1.5)
      \pscustom[linecolor=boundColor,fillstyle=solid,fillcolor=negColor]{
        \rotate{32}
        \psellipse(-0.1,1.5)(.5,1.1)
      }
      \lpcontourplot[linecolor=innerBallCol,fillstyle=solid,fillcolor=innerBallCol]{1}{0.5}
      \lpcontourplot[linecolor=green]{1}{0.7}
      \placeTarget
      \rput[lt]{*0}(-1.0,1.3){\large $\xminus$}
      \rput[lt]{*0}(0.2,1.3){\large $\xplus$}
      \endpspicture \\
      \refstepcounter{subfig}
      \thesubfig
      \label{fig:convexNeg}
    \end{center}
  \end{minipage}
  \caption{Geometry of convex sets and \LP[1] balls. %
    \ref{fig:convexPos} If the positive set \xplus\ is convex, finding an \LP[1] ball
    contained within \xplus establishes a lower bound on the cost, otherwise
    at least one of the \LP[1] ball's corners witnesses an upper bound.
    \ref{fig:convexNeg} If the negative set \xminus\ is convex, we can establish upper and
    lower bounds on the cost by determining whether or not an \LP[1] ball
    intersects with \xminus, but this intersection need not include
    any corner of the ball.}
  \label{fig:convexClassifiers}
\end{figure}

We generalize \kSearchability\ to the family of \term{\convexClass}
$\classSpace^\mathrm{convex}$ that partition the
feature space $\xspace$ into a positive and negative class, one of
which is convex.  The \convexClass\ include the linear classifiers
studied by \citet{lowd05adversarial},
anomaly detectors using bounded PCA~\citep{LCD04} and that
use hyper-sphere boundaries~\citep{PRML},
one-class classifiers that predict anomalies by thresholding the
log-likelihood of a log-concave (or uni-modal) density function, and
quadratic classifiers of the form $\dpt^\top \mat{A} \dpt +
\vec{b}^\top \dpt + c \ge 0$ if $\mat{A}$ is semidefinite.  The
\convexClass\ also include complicated bodies such as any
intersections of a countable number of halfspaces, cones, or balls.

Restricting $\classSpace$ to be the family of \convexClass\ simplifies
\kIMAC\ search. When the negative class $\xminus$ is convex, the
problem reduces to minimizing a (convex) function $\adCost$
constrained to a convex set---if $\xminus$ were known to the
adversary, this simply corresponds to solving a convex program.
When the positive class $\xplus$ is convex, however, our task is to
minimize the (convex) function $\adCost$ outside of a convex set; this
is generally a hard problem (\cf\ Section~\ref{sec:MLS-L2} where we show
that minimizing \LP[2] cost can require exponential query complexity). 
Nonetheless for certain cost functions \adCost, it
is easy to determine whether a particular cost ball
$\ball{C}{\adCost}$ is completely contained within a convex set. This
leads to efficient approximation algorithms.

We construct efficient algorithms for query-based optimization of the
\LP[1] cost of Eq.~\eqref{eq:weightedL1} for the \convexClass.  There
appears to be an asymmetry depending on whether the positive or
negative class is convex as illustrated in
Figure~\ref{fig:convexClassifiers}. When the positive set is convex,
determining whether an \LP[1] ball
$\ballNoStretch{C}{\lpCost[\featCost]{1}} \subset \xplus$ only
requires querying the vertices of the ball as depicted in
Figure~\subref{fig:convexClassifiers}{fig:convexPos}. When the
negative set is convex, determining whether or not
$\ballNoStretch{C}{\lpCost[\featCost]{1}} \cap \xminus = \emptyset$ is
non-trivial since the intersection need not occur at a vertex as
depicted in Figure~\subref{fig:convexClassifiers}{fig:convexNeg}. We
present an efficient algorithm for the optimizing a \LP[1] cost when
$\xplus$ is convex and a polynomial random algorithm for optimizing
any convex cost when $\xminus$ is convex.

The algorithms we present achieve multiplicative optimality via binary
search. We use Eq.~\eqref{eq:Lmult} to define \steps\ as the number of
phases required by our binary search to reduce the multiplicative gap
to less than $1+\epsilon$. We also use $\maxCost =
\lpCostFunc[\featCost]{1}{\dpt^-}$ as an initial upper bound on the
\MAC\ and assume there is some $\minCost > 0$ that lower bounds the
\MAC\ (\ie\ $\xtarget$ is in the interior of $\xplus$).  This condition
eliminates the case where $\xtarget$
is on the boundary of $\xplus$ where
$\function{\MAC}{\classifier,\adCost}=0$ and
$\function{\kIMAC}{\classifier,\adCost}=\emptyset$---in this
degenerate case, no algorithm can find an \kIMAC\ since there are
negative instances arbitrarily close to $\xtarget$.

\subsection{\kIMAC\ Search for a Convex $\xplus$}
\label{sec:pos-convex}

Solving the \kIMAC\ Search problem when $\xplus$ is hard in
the general case of convex cost $\adCostFunc{\cdot}$. We demonstrate
algorithms for the \LP[1] cost of Eq.~\eqref{eq:weightedL1}
that solve the problem as a binary search. Namely, given initial costs
\minCost\ and \maxCost\ that bound the \MAC, our algorithm can
efficiently determine whether $\ball{C}{\lpCost{1}} \subset \xplus$
for any intermediate cost $\minCost[t] < C_t < \maxCost[t]$. If the
\LP[1] ball is contained in $\xplus$, then $C_t$ becomes the new lower
bound $\minCost[t+1]$. Otherwise $C_t$ becomes the new upper bound
$\maxCost[t+1]$.  Since our objective Eq.~\eqref{eq:kIMAC} is to
obtain multiplicative optimality, our steps will be
$C_t=\sqrt{\minCost[t]\cdot\maxCost[t]}$. We now explain how we
exploit the properties of the \LP[1] ball and convexity of
$\xplus$ to efficiently determine whether $\ball{C}{\lpCost{1}}
\subset \xplus$ for any $C$. We also discuss practical aspects of our
algorithm and extensions to other \LP\ cost functions.

The existence of an efficient query algorithm relies on three facts:
(1) $\xtarget \in \xplus$; (2) every \LP[1] cost $C$-ball
centered at $\xtarget$ intersects with $\xminus$ only if at least one
of its vertices is in $\xminus$; and (3) $C$-balls of \LP[1]
costs only have $2\cdot \dims$ vertices. The vertices of the 
\LP[1] ball $\ball{C}{\lpCost{1}}$ are axis-aligned instances
differing from $\xtarget$ in exactly one feature (\eg\ the
\nth{d}{th} feature) and can be expressed in the form
\begin{equation}
  \xtarget \pm \frac{C}{\ithCost{d}} \coordinVect{d} \enspace,
  \label{eq:optAxisVect}
\end{equation}
which belongs to the $C$-ball of our \LP[1] cost (the
coefficient $\tfrac{C}{\ithCost{d}}$ normalizes for the weight
$\ithCost{d}$ on the \nth{d}{th} feature). We now formalize the second
fact as follows.

\begin{lemma}
  \label{thm:optAxis}
  For all $C > 0$, if there exists some $\dpt \in \xminus$ that
  achieves a cost of $C=\lpCostFunc[\featCost]{1}{\dpt}$, then there
  is some feature $d$ such that a vertex of the form of
  Eq.~\eqref{eq:optAxisVect} is in $\xminus$ (and also achieves cost
  $C$ by Eq.~\ref{eq:weightedL1}).
\end{lemma}
\begin{proof}
  Suppose not; then there is some $\dpt \in \xminus$ such that
  $\lpCostFunc[\featCost]{1}{\dpt}=C$ and $\dpt$ has $M\ge2$ features
  that differ from $\xtarget$ (if $\dpt$ only differs in $1$ feature
  it would be of the form of Eq.~\ref{eq:optAxisVect}). Let
  $\set{d_1,\ldots,d_M}$ be the differing features and let $b_{d_i} =
  \sign\left(\dptc_{d_i} - \xtargetc_{d_i}\right)$ be the sign of the
  difference between $\dpt$ and $\xtarget$ along the $d_i$-th feature.
  For each $d_i$, let $\vec{e}_{d_i} = \xtarget +
  \frac{C}{\ithCost{d_i}} \cdot b_{d_i} \cdot \coordinVect{d_i}$ be a
  vertex of the form of Eq.~\eqref{eq:optAxisVect} which has a cost
  $C$ (from Eq.~\ref{eq:weightedL1}).  The $M$ vertices
  $\vec{e}_{d_i}$ form an $M$-dimensional equi-cost simplex of cost
  $C$ on which $\dpt$ lies; \ie\ $\dpt = \sum_{i=1}^{M}{\alpha_i
    \vec{e}_{d_i}}$ for some $0 \le \alpha_i \le 1$. If all
  $\vec{e_{d_i}} \in \xplus$, then the convexity of $\xplus$ implies
  that all points in their simplex are in $\xplus$ and so $\dpt \in
  \xplus$ which violates our premise. Thus, if any instance in
  $\xminus$ achieves cost $C$, there is always a vertex of the form
  Eq.~\eqref{eq:optAxisVect} in $\xminus$ that also achieves cost $C$.
\end{proof}

As a consequence, if all such vertices of any $C$ ball
$\ball{C}{\lpCost{1}}$ are positive, then all $\dpt$ with
$\lpCost[\featCost]{1}{\dpt} \le C$ are positive thus establishing $C$
as a lower bound on the \MAC. Conversely, if any of the vertices of
$\ball{C}{\lpCost{1}}$ are negative, then $C$ is an upper bound on
\MAC. Thus, by simultaneously querying all $2\cdot \dims$ equi-cost
vertices of $\ball{C}{\lpCost{1}}$, we either establish $C$ as a new
lower or upper bound on the \MAC. By performing a binary search on $C$
we iteratively halve the multiplicative gap between our bounds until
it is within a factor of $1+\multGoal$. This yields an \kIMAC\ of the
form of Eq.~\eqref{eq:optAxisVect}.

\begin{figure}[t]
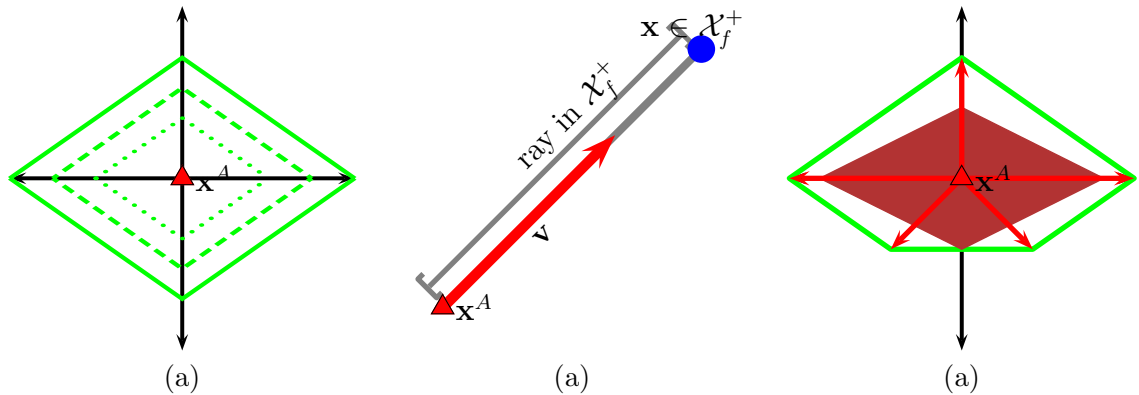

\begin{minipage}{.32\linewidth}
\begin{center}
\psset{unit=0.47\linewidth}
\pspicture(-1,-1)(1,1)
  \psaxes[ticks=none,labels=none,linewidth=1.5pt]{<->}(0,0)(-1,-1)(1,1)
  \lpcontourplot[linecolor=green,linewidth=2pt,yweight=0.7]{1}{1.0}
  \lpcontourplot[linecolor=green,linestyle=dashed,linewidth=2pt,yweight=0.7]{1}{0.75}
  \lpcontourplot[linecolor=green,linestyle=dotted,linewidth=2pt,yweight=0.7]{1}{0.5}
  \placeTarget
\endpspicture \\
\refstepcounter{subfig}
\thesubfig
\label{fig:LPballs}
\end{center}
\end{minipage} \hfill
\begin{minipage}{.32\linewidth}
\begin{center}
\psset{unit=0.47\linewidth}
\pspicture*(-.25,-.25)(1.75,1.75)
  \psline[linecolor=gray,linewidth=2pt]{[-]}(-.1,.1)(1.4,1.6)
  \rput[rb]{*45}(1.1,1.3){\large ray in $\xplus$}
  \psline[linecolor=gray,linewidth=3pt]{}(0,0)(1.5,1.5)
  \pscircle*[linecolor=blue](1.5,1.5){.08}
  \rput[rb]{*0}(1.75,1.5){\large $\dpt\in\xplus$}
  \psline[linecolor=red,linewidth=4pt]{->}(0,0)(1,1)
  \uput{5pt}[-45]{45}(.5,.5){\large $\vec{v}$}
  \placeTarget
\endpspicture \\
\thesubfig
\label{fig:searchRay}
\end{center}
\end{minipage} \hfill
\begin{minipage}{.32\linewidth}
\begin{center}
\psset{unit=0.47\linewidth}
\pspicture(-1,-1)(1,1)
  \psaxes[ticks=none,labels=none,linewidth=1.5pt]{<->}(0,0)(-1,-1)(1,1)
  \pspolygon[linecolor=green,linewidth=2pt](1,0)(0,.7)(-1,0)(-.4118,-.4118)(.4118,-.4118)
  \lpcontourplot[linecolor=innerBallCol,fillstyle=solid,fillcolor=innerBallCol,yweight=0.5]{1}{0.8236}
  \psline[linecolor=red,linewidth=2pt]{->}(0,0)(1,0)
  \psline[linecolor=red,linewidth=2pt]{->}(0,0)(-1,0)
  \psline[linecolor=red,linewidth=2pt]{->}(0,0)(0,0.7)
  \psline[linecolor=red,linewidth=2pt]{->}(0,0)(-.4118,-.4118)
  \psline[linecolor=red,linewidth=2pt]{->}(0,0)(.4118,-.4118)
  \placeTarget
\endpspicture \\
\thesubfig
\label{fig:boundedBall}
\end{center}
\end{minipage}
  \caption{The geometry of search. \ref{fig:LPballs} Weighted \LP[1] balls are centered around the target $\xtarget$ and have $2^\dims$ vertices; \ref{fig:searchRay} Search directions in multi-line search radiate from $\xtarget$ to probe specific costs; \ref{fig:boundedBall} In general, we leverage convexity of the cost function when searching to evade. By probing all search directions at a specific cost, the convex hull of the positive queries bounds the \LP[1] cost ball contained within it.}
\label{fig:mls}
\end{figure}

A general form of this multiline search procedure is presented as
Algorithm~\ref{alg:mls} and depicted in Figure~\ref{fig:mls}. \MLS\ simultaneously searches along the
directions in a set $\aset{W}$ of search directions that radiate from
their origin at $\xtarget$ and that are unit vectors for their cost;
\ie\ $\adCostFunc{\vec{w}} = 1$ for any $\vec{w} \in \aset{W}$. (We
transform a given set of non-normalized search vectors $\set{\vec{v}}$
into unit search vectors by simply applying a normalization constant
of $\adCostFunc{\vec{v}}^{-1}$ to each vector.)  At each step of \MLS,
at most $|\aset{W}|$ queries are issued in order to construct a
bounding shell (\ie\ the convex hull of these queries will either form
an upper or lower bound on the \MAC) to determine whether
$\ball{C}{\adCost} \subset \xplus$. Once a negative instance is found
at cost $C$, we cease further queries at cost $C$ since a single
negative instance is sufficient to establish a lower bound. We call
this policy \term{lazy querying}\footnote{We could continue querying
  at any distance $B^-$ where there is a known negative instance as it
  may allow us to prune other search directions quickly. However, once
  the classifier reveals a negative instance at distance $B^-$, the
  classifier would be foolish to subsequently reveal that another
  direction has a \posLbl\ at the same distance since it freely allows
  the adversary to prune a search direction. Hence, a \malcal\ will
  always respond with \negLbl\ for any cost where a negative instance
  has already been revealed. Thus, our algorithm uses lazy querying
  and only queries at costs below our upper bound $\maxCost[t]$ on the
  \MAC.}.  Further, when an upper bound is established for a cost $C$
(a negative vertex is found), our algorithm prunes all directions that
were positive at cost $C$. This pruning is sound; by the convexity
assumption these pruned directions are positive for all costs less
than the new upper bound $C$ on the \MAC. Finally, by performing a
binary search on the cost, \MLS\ finds a \kIMAC\ with no more than
$|\aset{W}| \cdot \steps$ queries but at least $|\aset{W}|+\steps$
queries.  Thus, this algorithm is $\bigO{|\aset{W}| \cdot \steps}$ for
\LP[1] costs.

It is worth noting that, in its present form, \MLS\ has two implicit
assumptions. First, we assume all search directions radiate from a
common origin, \xtarget, and $\adCostFunc{\xtarget} = 0$. Without this
assumption, the ray-constrained cost function $\adCostFunc{\xtarget +
  s \cdot \vec{w}}$ is still convex in $s \ge 0$ but not necessarily
monotonic as required for binary search.  Second, we assume the cost
function \adCost\ is a \term{positive homogeneous function} along an
ray from \xtarget; \ie\ $\adCostFunc{\xtarget + s \cdot \vec{w}} = |s|
\cdot \adCostFunc{\xtarget + \vec{w}}$. This assumption allows \MLS\
to scale its unit search vectors to achieve the same scaling of their
cost. Although the algorithm could be adapted to eliminate these
assumptions, the cost functions in Eq.~\eqref{eq:weightedL1} satisfy
both assumptions since they are norms centered at \xtarget.

Algorithm~\ref{alg:convexPlus} uses \MLS\ for
\LP[1] costs by making $\aset{W}$ be the vertices of the unit-cost
\LP[1] ball centered at $\xtarget$. In this case, the search issues at
most $2\cdot\dims$ queries to determine whether $\ball{C}{\lpCost{1}}
\subset \xplus$ and so Algorithm~\ref{alg:convexPlus} is
$\bigO{\steps\cdot\dims}$. However, \MLS\ does not rely on its
directions being vertices of the \LP[1] ball although those vertices
are sufficient to span the \LP[1] ball. Generally, \MLS\ is agnostic
to the configuration of its search directions and can be adapted for
any set of directions that can provide a bound on the cost using the
convexity of \xplus. However, as we show in Section~\ref{sec:LPCosts},
the number of search directions required to bound an \LP[p] for $p>1$
can be exponential in $\dims$.

\begin{figure*}[t]
\begin{minipage}{0.45\linewidth}
  \begin{Algorithm}{Multi-line Search}{alg:mls}
    \STATE $\function{MLS}{\aset{W},\xtarget,\dpt^-,C^+_0,C^-_0,\multGoal}$
    \STATE $\dpt^{\ast} \gets \dpt^{-}$
    \STATE $t \gets 0$
    \WHILE{$C^-_t / C^+_t > 1+\multGoal$}
      \STATE $C_t \gets \sqrt{C^+_t \cdot C^-_t}$
      \FORALL{$\vec{e} \in \aset{W}$}
        \STATE \textbf{Query}: $f_\vec{e}^t \gets \prediction{\xtarget + C_t \cdot \vec{e}}$
        \IF{$f_\vec{e}^t = \negLbl$}
          \STATE $\dpt^{\ast} \gets \xtarget + C_t \cdot \vec{e}$
          \STATE Prune $\vec{i}$ from $\aset{W}$ if $f_\vec{i}^t = \posLbl$
          \STATE \textbf{break for-loop}
        \ENDIF
      \ENDFOR
      \STATE $C^+_{t+1} \gets C^+_t$ and $C^-_{t+1} \gets C^-_t$
      \STATE \textbf{if} $\forall \vec{e} \in \aset{W} \; f_\vec{e}^t = \posLbl$ \textbf{then} $C^+_{t+1} \gets C_t$
      \STATE \textbf{else} $C^-_{t+1} \gets C_t$
      \STATE $t \gets t+1$
    \ENDWHILE
    \STATE \textbf{return:} $\dpt^{\ast}$
  \end{Algorithm}
\end{minipage}
\hspace{2em}
\begin{minipage}{0.45\linewidth}
  \begin{minipage}{\linewidth}
  \begin{Algorithm}{Convex $\xplus$ Set Search}{alg:convexPlus}
    \STATE $\function{ConvexSearch}{\aset{W},\xtarget,\dpt^-,\multGoal,C^+}$
    \STATE $C^- \gets \adCostFunc{\dpt^-}$
    \STATE $\aset{W} \gets \emptyset$
    \FORALL{$i \in 1\ldots\dims$}
      \STATE $\vec{e}^{i} \gets \frac{1}{\ithCost{i}}  \cdot \coordinVect{i}$
      \STATE $\aset{W} \gets \aset{W} \cup \set{\pm \vec{e}^i}$
    \ENDFOR
    \STATE \textbf{return:} $\function{MLS}{\aset{W},\xtarget,\dpt^-,C^+,C^-,\multGoal}$ 
  \end{Algorithm}
  \end{minipage} \\ \vspace{2em}

  \begin{minipage}{\linewidth}
  \begin{Algorithm}{Linear $\xplus$ Set Search}{alg:linear}
    \STATE $\function{LinearSearch}{\aset{W},\xtarget,\dpt^-,\multGoal,C^+}$
    \STATE $C^- \gets \adCostFunc{\dpt^-}$
    \STATE $\aset{W} \gets \emptyset$
    \FORALL{$i \in 1\ldots\dims$}
      \STATE $\vec{e}^{i} \gets \frac{1}{\ithCost{i}} \cdot \coordinVect{i}$
      \STATE $b_i \gets \sign\left(\dptc^-_i - \xtargetc_i\right)$
      
      \STATE \textbf{if} $b_i = 0$ \textbf{then} $\aset{W} \gets \aset{W} \cup \set{b_i \vec{e}^{i}}$
      \STATE \textbf{else} $\aset{W} \gets \aset{W} \cup \set{\pm \vec{e}^i}$
    \ENDFOR
    \STATE \textbf{return:} $\function{MLS}{\aset{W},\xtarget,\dpt^-,C^+,C^-,\multGoal}$ 
  \end{Algorithm}
  \end{minipage}%
\end{minipage}
\end{figure*}

\subsubsection{$K$-step Multi-Line Search}

Here we present a variant of the multi-line search algorithm that
better exploits pruning to reduce the query complexity of
Algorithm~\ref{alg:mls}---we call this variant \KMLS.  The \MLS\
algorithm is $2\cdot |\aset{W}|$ simultaneous binary searches
(breadth-first). This strategy prunes directions most effectively when
the convex body is assymetrically elongated relative to \xtarget\ but
fails to prune for symmetrically rounded bodies. Instead we could
search each direction sequentially (depth-first) and still obtain a
worst case of $\bigO{\steps\cdot\dims}$ queries. In contrast, this
strategy reduces queries used to shrink the cost gap on symmetrically
rounded bodies but is unable to do so for assymetrically elongated
bodies.
We therefore propose an algorithm that mixes these strategies.

At each phase, the \KMLS\ (Algorithm~\ref{alg:kmls}) chooses a single
direction $\vec{e}$ and queries it for $K$ steps to generate candidate
bounds $B^-$ and $B^+$ on the \MAC. The algorithm makes substantial
progress towards reducing $G_t$ without querying other directions
(depth-first). It then iteratively queries all remaining directions at
the candidate lower bound $B^+$ (breadth-first).  Again we use lazy
querying and stop as soon as a negative instance is found since $B^+$
is then no longer a viable lower bound. In this case, although the
candidate bound is invalidated, we can still prune all directions that
were positive at $B^+$. Thus, in every iteration, either the gap is
decreased or at least one search direction is pruned. We show that for
$K=\lceil \sqrt{\steps} \rceil$, the algorithm achieves a delicate balance
between breadth-first and depth-first approaches to attain a better
worst-case complexity than either.

\begin{figure}[t]
\begin{center}
\begin{minipage}{0.65\linewidth}
  \begin{Algorithm}{$K$-Step Multi-line Search}{alg:kmls}
    \STATE $\function{KMLS}{\aset{W},\xtarget,\dpt^-,C^+_0,C^-_0,\multGoal,K}$
    \STATE $\dpt^{\ast} \gets \dpt^{-}$
    \STATE $t \gets 0$
    \WHILE{$C^-_t / C^+_t > 1+\multGoal$}
      \STATE Choose a direction $\vec{e} \in \aset{W}$
      \STATE $B^+ \gets C^+_t$
      \STATE $B^- \gets C^-_t$
      \FOR{$K$ steps}
        \STATE $B \gets \sqrt{B^+ \cdot B^-}$
        \STATE \textbf{Query}: $f_\vec{e} \gets \prediction{\xtarget + B \cdot \vec{e}}$
        \STATE \textbf{if} $f_\vec{e} = \posLbl$ \textbf{then} $B^+ \gets B$
        \STATE \textbf{else} $B^- \gets B$ \textbf{and} $\dpt^{\ast} \gets \xtarget + B \cdot \vec{e}$
      \ENDFOR

      \FORALL{$\vec{i} \neq \vec{e} \in \aset{W}$}
        \STATE \textbf{Query}: $f_\vec{i}^t \gets \prediction{\xtarget + (B^+) \cdot \vec{i}}$
        \IF{$f_\vec{i}^t = \negLbl$}
          \STATE $\dpt^{\ast} \gets \xtarget + (B^+) \cdot \vec{i}$
          \STATE Prune \vec{k} from $\aset{W}$ if $f_\vec{k}^t = \posLbl$
          \STATE \textbf{break for-loop}
        \ENDIF
      \ENDFOR

      \STATE $C^-_{t+1} \gets B^-$
      \STATE \textbf{if} $\forall \vec{i} \in \aset{W} \; f_\vec{i}^t = \posLbl$ \textbf{then} $C^+_{t+1} \gets B^+$
      \STATE \textbf{else} $C^-_{t+1} \gets B^+$
      \STATE $t \gets t+1$
    \ENDWHILE
    \STATE \textbf{return:} $\dpt^{\ast}$
  \end{Algorithm}
\end{minipage}
\end{center}
\end{figure}

\begin{theorem}
  \label{thm:sqrtL}
  Algorithm~\ref{alg:kmls} will find an \kIMAC\ with at most
  $\bigO{\steps + \sqrt{\steps} |\aset{W}|}$ queries when $K = \lceil \sqrt{\steps} \rceil$.
\end{theorem}

The proof of this theorem appears in Appendix~\ref{app:MLS-proofs}.
As a consequence of Theorem~\ref{thm:sqrtL}, finding a \kIMAC\ with
Algorithm~\ref{alg:kmls} for a \LP[1] cost requires
$\bigO{\steps + \sqrt{\steps}\dims}$ queries. Further, both
Algorithms~\ref{alg:convexPlus} and~\ref{alg:linear} can incorporate
\KMLS\ directly by replacing their function
call to $\mathrm{MLS}$ to $\mathrm{KLMS}$ and using $K =
\lceil\sqrt{\steps}\rceil$.

\subsubsection{Lower Bound}
\label{sec:lowBounds}

Here we find lower bounds on the number of queries required by any
algorithm to find an \kIMAC\ when $\xplus$ is convex for any convex
cost function (\eg\ Eq.~\ref{eq:weightedL1} for $p \ge 1$).  Below we
present two theorems, one for both additive and multiplicative
optimality.  Notably, since an \kIMAC\ uses multiplicative optimality,
we incorporate a lower bound $\minCost > 0$ on the \MAC\ into our
statement.

\begin{theorem}
  \label{thm:alower}
  For any $\dims>0$, any positive convex function $\adCost:
  \reals^\dims \to \realpos$, any initial bounds $0 \le \minCost <
  \maxCost$ on the \MAC, and $0 < \addGoal < \maxCost - \minCost$, all
  algorithms must submit at least $\max\{\dims,\addSteps\}$ membership
  queries in the worst case to be \addGoal-additive optimal on
  $\classSpace^\mathrm{convex,\posLbl}$.
\end{theorem}

\begin{theorem}
  \label{thm:mlower}
  For any $\dims>0$, any positive convex function $\adCost: \reals^\dims
  \to \realpos$, any initial bounds $0 < \minCost < \maxCost$ on the \MAC, and $0 < \multGoal < \frac{\maxCost}{\minCost} - 1$, all
  algorithms must submit at least $\max\{\dims,\multSteps\}$
  membership queries in the worst case to be
  \multGoal-multiplicatively optimal on
  $\classSpace^\mathrm{convex,\posLbl}$.
\end{theorem}

The proof of both of these theorems is in
Appendix~\ref{app:lowerBounds}. In these theorems, we restrict
\addGoal\ and \multGoal\ to the intervals $\left(0,\maxCost -
  \minCost\right)$ and $\left(0,\frac{\maxCost}{\minCost} - 1\right)$
respectively. In fact, outside of these intervals the query strategies
are trivial. For either $\addGoal=0$ or $\multGoal=0$ no approximation
algorithm will terminate and for $\addGoal \ge \maxCost - \minCost$ or
$\multGoal \ge \frac{\maxCost}{\minCost} - 1$, $\dpt^-$ is an \IMAC,
so no queries are required.

Theorem~\ref{thm:alower} and~\ref{thm:mlower} show that one needs that
\addGoal-additive and \multGoal-multiplicative optimality require
$\bigOmegaNoStretch{\addSteps + \dims}$ and
$\bigOmegaNoStretch{\multSteps + \dims}$ queries respectively. Thus,
we see that our \KMLS\ algorithm (Algorithm~\ref{alg:kmls}) has close
to the optimal query complexity for \LP[1]-costs with its
$\bigONoStretch{\steps+\sqrt{\steps}\dims}$ queries. These results
also hold for arbitrary \LP[p] ($p \ge 1$) costs but we show lower
bounds in Section~\ref{sec:LPCosts} for $p > 1$ that substantially
exceed these results.

\subsubsection{Special Cases}
\label{sec:specialPosCases}

Here we present a number of special cases that require minor
modifications to Algorithms~\ref{alg:mls} and~\ref{alg:kmls} primarily
as preprocessing steps.

\paragraph{Revisiting Linear Classifiers}

Lowd and Meek originally developed a method for reverse engineering
linear classifiers for a \LP[1] cost. First their method isolates a
sequence of points from $\dpt^-$ to $\xtarget$ that cross the
classifier's boundary and then it estimates the hyperplane's
parameters using $\dims$ line searches. However, as a consequence of
the ability to efficiently minimize our objective when $\xplus$ is
convex, we immediately have an alternative method for linear
classifiers (\ie\ half-spaces). In fact, for this special case, as
many as half of the search directions can be eliminated using the
initial orientation of the hyperplane separating $\xtarget$ and
$\dpt^-$. Intuitively, the minimizer in the negative halfspace can
only occur along one of the axes of the orthants that contain
$\dpt^-$. This algorithm is presented as Algorithm~\ref{alg:linear}.
Moreover, because linear classifiers are a special case of
\convexClass, our \KMLS\ algorithm improves on the reverse-engineering
technique's $\bigO{\steps\cdot\dims}$ queries and applies to a broader
family.

\paragraph{Extending \MLS\ algorithms to $\ithCost{d} = \infty$ or $\ithCost{d} = 0$}

In Algorthms~\ref{alg:convexPlus} and~\ref{alg:linear}, we reweighted
the \nth{d}{th} axis-aligned directions by a factor
$\frac{1}{\ithCost{d}}$ to make unit cost vectors but implictly
assuming $\ithCost{d} \in \left( 0, \infty \right)$. The case where
$\ithCost{d} = \infty$ (\eg\ immutable features) is dealt with simply
removing those features from the set of search directions \aset{W}
used in the \MLS. 
In the case when $\ithCost{d} = 0$ (\eg\ useless features), \MLS-like
algorithms no longer ensure near-optimality because they implicitly
assume that cost balls are bounded sets.  If $\ithCost{d} = 0$,
$\ball{0}{\adCost}$ is no longer bounded and a $0$-cost could be
achieved if $\xminus$ anywhere intersects the subspace spanned by the
$0$-cost features---this makes near-optimality unachievable unless a
negative $0$-cost instance can be found. In the worst case, such an
instance could be arbitrarily far in any direction within the $0$-cost
subspace making search for such an instance intractable. Nonetheless,
one possible search strategy is to assign all $0$-cost features a
non-zero weight that decays quickly toward $0$ (\eg\ $\ithCost{d} =
2^-t$ in the \nth{t}{th} iteration) as we repeatly rerun an \MLS\ on
the altered objective for $T$ iterations.  We will either find a
negative instance that only alters $0$-cost features (and hence is a
\kIMAC[0]), or we will terminate assuming no such instance exists.
This algorithm does not ensure near-optimality but may find a suitable
instance with only $T$ runs of a \MLS.

\paragraph{Lack of an Initial Lower Bound}

Thus far, to find a \kIMAC\ our algorithms have searched between
initial bounds \minCost\ and \maxCost, but, in general, \minCost\ may
not be known to a real-world adversary. We now present an algorithm we
call \algoName{SpiralSearch} that can efficiently establish a lower
bound on the \MAC\ if one exists. This algorithm performs a halving
search on the exponent along a single direction to find a positive
example, then queries the remaining directions at that cost. Either
the lower bound is verified or directions that were positive can be
pruned for the remainder of the search.

\begin{center}
\begin{minipage}{0.65\linewidth}
  \begin{Algorithm}{Spiral Search}{alg:spiral}
    \STATE $\function{spiral}{\aset{W},\xtarget,\dpt^-,\maxCost,\multGoal}$
    \STATE $t \gets 0$ and $\aset{V} \gets \emptyset$
    \REPEAT
      \STATE Choose a direction $\vec{e} \in \aset{W}$
      \STATE Remove $\vec{e}$ from $\aset{W}$ and $\aset{V} \gets \aset{V} \cup \set{\vec{e}}$
      \STATE \textbf{Query}: $f_\vec{e} \gets \prediction{\xtarget + (\maxCost) 2^{-2^{t}} \vec{e}}$
      \IF{$f_\vec{i} = \negLbl$}
        \STATE $\aset{W} \gets  \aset{W} \cup \set{\vec{e}}$ and $\aset{V} \gets \emptyset$
        \STATE $t \gets t+1$
      \ENDIF
    \UNTIL{$\aset{W} = \emptyset$}
    \STATE $\minCost \gets \maxCost \cdot 2^{-2^{t}}$
    \STATE \textbf{return:} ($\aset{V}$,\minCost,\maxCost)
  \end{Algorithm}
\end{minipage}
\end{center}

At the \nth{t}{th} iteration of \algoName{SpiralSearch} a direction is
selected and queried at the current lower bound of $(\maxCost)
2^{-2^{t}}$. If the query is positive, that direction is added to the
set $V$ of directions consistent with the lower bound. Otherwise, all
directions in $V$ are discarded and the lower bound is lowered with an
exponentially decreasing exponent. Thus, given that some lower bound
$\minCost>0$ does exist, one will be found in $\bigO{\steps + \dims}$
queries and this algorithm can be used as a precursor to any of the
previous searches\footnote{If no lower bound on the cost exists, no
  algorithm can find a \kIMAC. As presented, this algorithm would not
  terminate, but in practice the search would be terminated after
  sufficiently many iterations.}. Further, the search directions
pruned by \algoName{SprialSearch} are also invalid for the subsequent
\MLS\ so the set $\aset{V}$ returned by \algoName{SprialSearch} will
be used as the set $\aset{W}$ for the subsequent search.

\paragraph{Lack of a Negative Example}

Our algorithms can also naturally be adapted to the case when the
adversary has no negative example $\dpt^-$. This is accomplished by
querying \LP[1] balls of doubly exponentially increasing cost until a
negative instance is found. During the \nth{t}{th} iteration, we probe
along every search direction at a cost $(\minCost) 2^{2^t}$; either
all probes are positive (and we have a new lower bound) or at least
one is negative and we can terminate the search. Once a negative
example is located (having probed for $T$ iterations), we must have
$(\minCost) 2^{2^{T-1}} <\function{\MAC}{\classifier,\adCost} \le
(\minCost) 2^{2^T}$; thus, $T = \left\lceil \log_2 \log_2
  \frac{\function{\MAC}{\classifier,\adCost}}{\minCost} \right\rceil$.
We can subsequently perform \MLS\ with $\minCost = 2^{2^{T-1}}$ and
$\maxCost = 2^{2^{T}}$; \ie\ $\log_2 \gap{0} = 2^{T-1}$. This
precursor step requires at most $|\aset{W}| \cdot T$ queries to
initialize the \MLS\ algorithm with a gap such that $\steps =
\left\lceil (T-1) + \log_2\frac{1}{\log_2 (1+\multGoal)} \right\rceil$
according to Eq.~\eqref{eq:Lmult}.

If there is neither an initial upper bound or lower bound, we proceed
by probing each search direction at cost $1$ using an additional
$|\set{W}|$ queries---we will subsequently have either an upper or
lower bound and can proceed accordingly.

\subsection{\kIMAC\ Learning for a Convex $\xminus$}

In this section, we consider minimizing a convex cost function
$\adCost$ (we focus on weighted \LP[1] costs in
Eq.~\ref{eq:weightedL1}) when the feasible set $\xminus$ is convex.
Any convex function can be efficiently minimized within a known convex
set (\eg\ using the Ellipsoid Method and Interior Point
methods; see \citealt{ConvexOptimization}). However, in our problem the
convex set is only accessible via membership queries.  We use a
randomized polynomial algorithm of 
\citet{bertsimas04convexrw} to minimize the cost function $\adCost$
given an initial point $\dpt^- \in \xminus$. For any fixed cost $C^t$
we use their algorithm to determine (with high probability) whether
$\xminus$ intersects with $\ball{C^t}{\adCost}$; \ie\ whether $C^t$ is
a new lower or upper bound on the \MAC. With high probability, we find
an \kIMAC\ in no more than \steps\ repetitions using binary search. We
now focus only on weighted \LP[1] costs (Eq.~\ref{eq:weightedL1})
and return to more general cases in Section~\ref{sec:genLPNegative}.

\begin{figure*}[t]
\begin{minipage}{0.55\linewidth}
\begin{Algorithm}{Intersect Search}{alg:intersect}
  \STATE $\function{IntersectSearch} {\aset{P}^0, \aset{Q} = \set{\dpt^j \in \aset{P}^0}, C}$
  \FORALL{$s = 1 \ldots T$}
  \STATE (1) Generate $2N$ samples $\set{\dpt^j}_{j=1}^{2N}$
  \STATE \quad Choose $\dpt$ from $\aset{Q}$
  \STATE \quad $\dpt^j \gets \function{HitRun}{\aset{P}^{s-1}, \aset{Q}, \dpt^j}$
  \STATE (2) If any $\dpt^j$, $\adCostFunc{\dpt^j} \le C$ terminate the for-loop
  \STATE (3) Put samples into 2 sets of size $N$ 
  \STATE \quad $\aset{R} \gets \set{\dpt^j}_{j=1}^{N}$ and $\aset{S} \gets \set{\dpt^j}_{j=2N+1}^{2N}$
  \STATE (4) $\vec{z}^{s} \gets \frac{1}{N}\sum_{\dpt^j\in\aset{R}}{\dpt^j}$
  \STATE (5) Compute $\aset{H}_{\vec{z}^{s}}$ using Eq.~\eqref{eq:halfspace}
  \STATE (6) $\aset{P}^{s} \gets \aset{P}^{s-1} \cap \aset{H}_{\vec{z}^{s}}$
  \STATE (7) Keep samples in $\aset{P}^s$
  \STATE \quad $\aset{Q} \gets \set{\dpt\in \aset{S} \land \dpt \in \aset{P}^s}$
  \ENDFOR
  \STATE \textbf{Return:} the found $[\dpt_j, \aset{P}^s, \aset{Q}]$; or No Intersect 
\end{Algorithm}
\end{minipage}\quad
\begin{minipage}{0.42\linewidth}
\begin{Algorithm}{Hit-and-Run}{alg:hitandrun}
  \STATE $\function{HitRun} {\aset{P}, \set{\vec{y}^j}, \dpt^0}$
  \FORALL{$i = 1 \ldots K$}
  \STATE (1) Choose a random direction: 
  \STATE \quad $\nu_j \sim \gauss{0}{1}$
  \STATE \quad $\vec{v} \gets \sum_j{\nu_j \cdot \vec{y}^j}$
  \STATE (2) Sample uniformly along $\vec{v}$ using rejection sampling:
  \STATE  Choose $\Omega$ s.t. $\dpt^{i-1} + \Omega \cdot \vec{v} \notin \aset{P}$ %
  \REPEAT
    \STATE $\omega \sim \function{Unif}{0,\Omega}$
    \STATE $\dpt^i \gets \dpt^{i-1} + \omega \cdot \vec{v}$
    \STATE $\Omega \gets \omega$
  \UNTIL{$\dpt^i \in \aset{P}$}
  \ENDFOR
  \STATE \textbf{Return:} $\dpt^K$
\end{Algorithm}
\end{minipage}
\end{figure*}

\subsubsection{Intersection of Convex Sets}

We now outline Bertsimas and Vempala's query-based procedure for
determining whether two convex sets (e.g., $\xminus$ and
$\ball{C^t}{\lpCost{1}}$) intersect. Their \algoName{IntersectSearch}
procedure (which we present as Algorithm~\ref{alg:intersect}) is a
randomized Ellipsoid method for determining whether there is an
intersection between two bounded convex sets: $\aset{P}$ is only
accessible through membership queries and $\aset{B}$ provides a
separating hyperplane for any point outside it. They use efficient
query-based approaches to uniformly sample from $\aset{P}$ to obtain
sufficiently many samples such that cutting $\aset{P}$ through the
centroid of these samples with a separating hyperplane from $\aset{B}$
will significantly reduce the volume of $\aset{P}$ with high
probability.  Their technique thus constructs a sequence of
progressively smaller feasible sets $\aset{P}^s \subset
\aset{P}^{s-1}$ until either the algorithm finds a point in $\aset{P}
\cap \aset{Q}$ or it is highly unlikely that the intersection is
non-empty.

Our problem reduces to finding the intersection between $\xminus$ and
$\ball{C^t}{\lpCost{1}}$. Though $\xminus$ may be unbounded, we are
minimizing a cost with bounded equi-cost balls, so we can instead use
the set $\aset{P}^0 = \xminus \cap \ball{2R}{\lpCost{1}}$ (where $R =
\adCostFunc{\dpt^-} > C^t$) is a (convex) bounded subset of $\xminus$
that envelops all of $\ball{C^t}{\lpCost{1}}$ and thus the
intersection $\xminus \cap \ball{C^t}{\lpCost{1}}$ if it exists. We
also assume that there is some $r>0$ such that there is an $r$-ball
contained in the convex set $\xminus$; \ie\ there exists
$\vec{y}\in\xminus$ such that $\ball{r}{\lpCost{1};\vec{y}} \subset
\xminus$.  We now detail this \algoName{IntersectSearch} procedure
(Algorithm~\ref{alg:intersect}).

The backbone of the algorithm is the capability to sample uniformly
from an unknown but bounded convex body by means of the
\algoName{hit-and-run} random walk technique introduced by
\citet{smith-96-hnr} (Algorithm~\ref{alg:hitandrun}). Given an
instance $\dpt^j \in \aset{P}^{s-1}$, \algoName{hit-and-run} selects a
random direction $\vec{v}$ through $\dpt^j$ (we return to the
selection of $\vec{v}$ in Section~\ref{sec:sampling}). Since
$\aset{P}^{s-1}$ is a bounded convex set, the set $\Omega =
\set[\dpt^j + \omega \vec{v} \in \aset{P}^{s-1}]{\omega > 0}$ is a
bounded interval indexing all feasible points along direction
$\vec{v}$ through $\dpt^j$. Sampling $\omega$ uniformly from $\Omega$
(using rejection sampling) yields the next step of the random walk;
$\dpt^j + \omega \vec{v}$.  Under the appropriate conditions (see
Section~\ref{sec:sampling}), the \algoName{hit-and-run} random walk
generates a sample uniformly from the convex body after
$\bigOstar{\dims^3}$ steps\footnote{$\bigOstar{\cdot}$ denotes the
  standard complexity notation $\bigO{\cdot}$ without logarithmic
  terms.}~\citep{lovasz04hnrfast}.

\paragraph{Randomized Ellipsoid Algorithm:} 

We use \algoName{hit-and-run} to obtain $2N$ samples $\set{\dpt^j}$
from $\aset{P}^{s-1} \subset \xminus$ for a single phase of the
randomized ellipsoid algorithm. If any sample $\dpt^j$ satisfies
$\lpCostFunc{1}{\dpt^j} \le C^t$, then $\dpt^j$ is in the intersection
of $\xminus$ and $\ball{C^t}{\lpCost{1}}$ and the procedure is
complete.  Otherwise, we want to significantly reduce the size of
$\aset{P}^{s-1}$ without excluding any of $\ball{C^t}{\lpCost{1}}$ so
that sampling concentrates toward the intersection (if it
exists)---for this we need a separating hyperplane for
$\ball{C^t}{\lpCost{1}}$. For any point $\vec{y} \notin
\ball{C^t}{\lpCost{1}}$, the (sub)gradient of the \LP[1] cost given by
\begin{equation}
  \label{eq:gradient}
  h_d^\vec{y} = \ithCost{d} \sign\left(y_d - \xtargetc_d\right) \enspace,
\end{equation}
and is a separating hyperplane for $\vec{y}$ and
$\ball{C^t}{\lpCost{1}}$.

To achieve efficiency, we choose a point $\vec{z} \in \aset{P}^{s-1}$
so that cutting $\aset{P}^{s-1}$ through $\vec{z}$ with the hyperplane
$\vec{h}^\vec{z}$ eliminates a significant fraction of
$\aset{P}^{s-1}$. To do so, $\vec{z}$ must be centrally located within
$\aset{P}^{s-1}$.  We use the
empirical centroid of the half of our samples in $\aset{R}$: $\vec{z}
= N^{-1}\sum_{\dpt \in \aset{R}}{\dpt}$ (the other half we will
be used in Section~\ref{sec:sampling}). We cut $\aset{P}^{s-1}$ with
the hyperplane $\vec{h}^\vec{z}$ through $\vec{z}$; \ie\ 
$\aset{P}^s = \aset{P}^{s-1} \cap \aset{H}_{\vec{z}}$ where
$\aset{H}_{\vec{z}}$ is the halfspace
\begin{equation}
\label{eq:halfspace}
\aset{H}_\vec{z} =
\set[\dpt^\top\vec{h}^\vec{z} \le \vec{z}^\top\vec{h}^\vec{z}]{\dpt}
\enspace .
\end{equation}
As shown by Bertsimas and Vempala, this cut achieves
$\function{vol}{\aset{P}^{s}} \le \frac{2}{3}
\function{vol}{\aset{P}^{s-1}}$ with high probability if $N =
\bigOstar{\dims}$ and $\aset{P}^{s-1}$ is near-isotropic (see
Section~\ref{sec:sampling}).
Since the ratio of
volumes between the initial circumscribing and inscribing balls of the
feasible set is $\left(R/r\right)^\dims$, the algorithm can
terminate after $T = \bigO{\dims \log(R/r)}$ unsuccessful iterations
with a high probability that the intersection is empty.

Because every
iteration in Algorithm~\ref{alg:intersect} requires $N=\bigOstar{D}$
samples, each of which need $K=\bigOstar{D^3}$ random walk steps, and
there are $\bigOstar{D}$ iterations, the total number of membership
queries required by Algorithm~\ref{alg:intersect} is $\bigOstar{D^5}$.

\subsubsection{Sampling from a Queriable Convex Body}
\label{sec:sampling}

In the randomized Ellipsoid algorithm, random samples are used for two
purposes: estimating the convex body's centroid and maintaining the
conditions required for the \algoName{hit-and-run} sampler to
efficiently generate points uniformly from a sequence of shrinking
convex bodies.  Until this point, we assumed the
\algoName{hit-and-run} random walk efficiently produces uniformly
random samples from any bounded convex body $\aset{P}$ accessible
through membership queries.  However, if the body is severely
elongated, randomly selected directions will rarely align with the
long axis of the body and our random walk will take small steps
(relative to the long axis) and mix slowly. For the sampler to mix
effectively, we need the convex body $\aset{P}$ to be sufficiently
round, or more formally \term{near-isotropic}; \ie\ for any unit
vector $\vec{v}$, $\expect{\dpt \sim
  \aset{P}}{\left(\vec{v}^\top\left(\dpt - \expect{\dpt \sim
        \aset{P}}{\dpt}\right)\right)^2}$ is bounded between $1/2$ and
$3/2$ of $\function{vol}{\aset{P}}$.

If the body is not near-isotropic, we must rescale $\xspace$ with an
appropriate affine transformation $\mat{T}$ so the resulting body
$\aset{P}^\prime$ is near-isotropic. With sufficiently many samples
from $\aset{P}$ we can estimate $\mat{T}$ as their empirical
covariance matrix. Instead, we rescale $\xspace$ implicitly using a
technique described by \citet{bertsimas04convexrw}. We maintain a set
$\aset{Q}$ of sufficiently many uniform samples from the body
$\aset{P}^s$ and in the \algoName{hit-and-run} algorithm
(Algorithm~\ref{alg:hitandrun}) we sample the direction $\vec{v}$
based on this set.  Intuitively, because the samples in $\aset{Q}$ are
distributed uniformly in $\aset{P}^s$, the directions we sample based
on the points in $\aset{Q}$ implicitly reflect the covariance
structure of $\aset{P}^s$. This is equivalent to sampling the
direction $\vec{v}$ from a normal distribution with zero mean the
covariance of $\aset{P}$.

We must ensure $\aset{Q}$ is a set of sufficiently many samples from
$\aset{P}^s$ after each cut: $\aset{P}^{s} \gets \aset{P}^{s-1} \cap
\aset{H}_{\vec{z}^{s}}$. To do so, we initially resample $2N$ points
from $\aset{P}^{s-1}$ using \algoName{hit-and-run}---half of these,
$\aset{R}$, are used to estimate the centroid $\vec{z}^s$ for the cut
and the other half, $\aset{S}$, are used to repopulate $\aset{Q}$
after the cut. Because $\aset{S}$ contains independent uniform samples
from $\aset{P}^{s-1}$, those in $\aset{P}^s$ after the cut constitute
independent uniform samples from $\aset{P}^s$ (\ie\ rejection
sampling).
By choosing $N$ sufficiently large, our cut will be sufficiently
deep and we will have sufficiently many points to resample $\aset{P}^s$
after the cut.

Finally, we also need an initial set $\aset{Q}$ of uniform samples
from $\aset{P}^0$ but, in our problem, we only have a single point
$\dpt^- \in \xminus$.  Fortunately, there is an iterative procedure
for putting the initial convex set $\aset{P}^0$ into a near-isotropic
position from which we obtain $\aset{Q}$.  The \algoName{RoundingBody}
algorithm described by \citet{sa_volume} uses $\bigOstar{D^4}$
membership queries to transforms the convex body into a near-isotropic
position. We use this as a preprocessing step for
Algorithms~\ref{alg:intersect} and~\ref{alg:setsearch}; that is, given
$\xminus$ and $\dpt^- \in \xminus$ we make $\aset{P}^0 = \xminus \cap
\ball{2R}{\lpCost{1};\dpt^-}$ and then use the \algoName{RoundingBody}
algorithm to produce an initial uniform sample $\aset{Q} = \set{\dpt^j
  \in \aset{P}^0}$. These sets are then the inputs to our search
algorithms.

\subsubsection{Optimization over \LP[1] Balls}
\label{sec:ballOpt}

We now revisit the outermost optimization loop (for searching the
minimum feasible cost) of the algorithm and suggest improvements.
First, since $\xtarget$, $\dpt^-$ and $\aset{Q}$ are the same for
every iteration of the optimization procedure, we only need to run the
\algoName{RoundingBody} procedure once as a preprocessing step rather
than running it as a preprocessing step every time
\algoName{IntersectSearch} is invoked. The set of samples $\set{\dpt^j
  \in \aset{P}^0}$ produced by \algoName{RoundingBody} are sufficient
to initialize the \algoName{IntersectSearch} at each stage of the
binary search over $C^t$. Second, the separating hyperplane
$\vec{h}_f^\vec{y}$ given by Eq.~\eqref{eq:gradient} does not depend
on the target cost $C^t$ but only on $\xtarget$, the common center of
all the \LP[1] balls. In fact, the separating hyperplane at point
$\vec{y}$ is valid for all \LP[1]-balls of cost $C <
\adCostFunc{\vec{y}}$. Further, if $C < C^t$, we have
$\ball{C}{\lpCost{1}} \subset \ball{C^t}{\lpCost{1}}$. Thus, the final
state from a successful call to \algoName{IntersectSearch} for the
$C^t$-ball as the starting state for any subsequent call to
\algoName{IntersectSearch} for all $C < C^t$.  These improvements are
reflected in our final procedure \algoName{SetSearch} in
Algorithm~\ref{alg:setsearch}---the total number of queries required
is also $\bigOstar{D^5}$. %

\begin{figure}
\begin{center}
  \begin{minipage}{0.7\linewidth}
  \begin{Algorithm}{Convex $\xminus$ Set Search}{alg:setsearch}
      \STATE $\function{SetSearch}{\aset{P}, \aset{Q}=\set{\dpt^j \in \aset{P}}, \maxCost, \minCost, \multGoal}$
      \STATE $\dpt^\ast \gets \dpt^-$ and $t \gets 0$
      \WHILE{$\maxCost[t] / \minCost[t] > 1+\multGoal$}
        \STATE $C_t \gets \sqrt{\maxCost[t] \cdot \minCost[t]}$
        \STATE $[\dpt^\ast, \aset{P}^\prime, \aset{Q}^\prime] \gets \function{IntersectSearch}{\aset{P},\aset{Q},C}$
        \IF{intersection found}
          \STATE $\maxCost[t+1] \gets \adCostFunc{\dpt^\ast}$ and $\minCost[t+1] \gets \minCost[t]$
          \STATE $\aset{P} \gets \aset{P}^\prime$ and $\aset{Q} \gets \aset{Q}^\prime$
        \ELSE
        \STATE $\maxCost[t+1] \gets \maxCost[t]$ and $\minCost[t+1] \gets C_t$
        \ENDIF
        \STATE $t \gets t+1$
      \ENDWHILE
      \STATE \textbf{Return:} $\dpt^\ast$
  \end{Algorithm}
  \end{minipage}
\end{center}
\end{figure}

\newcommand{\entropy}[1]{\function{H}{#1}}

\section{General \LP\ Costs}
\label{sec:LPCosts}

Here we further extend \kSearchability\ over the family of \convexClass\ to the full family of \LP[p] costs for any $0 < p
< \infty$. As we demonstrate in this section,
many \LP[p] costs are not generally \kSearchable\ for all $\multGoal >
0$ over the family of \convexClass (\ie\ we show that finding
an \kIMAC\ for this family can require exponentially many queries in
$\dims$ and \multGoal). In fact, only the weighted \LP[1] costs are
known to have (randomized) polynomial query strategies when either the
positive or negative set is convex.

\subsection{Convex Positive Set}
\label{sec:genLPPositive}

\newcommand{\simplexMin}{\dpt^\ast}

Here we explore the ability of \MLS\ and \KMLS\ algorithms presented
in Section~\ref{sec:pos-convex} to find solutions to the near-optimal
evasion problem for \LP[p] cost functions with $p \neq 1$.
Particularly for $p>1$ we will be exploring the consequences of using
the \MLS\ algorithms using more search directions than just the $2
\cdot \dims$ axis-aligned directions.  Figure~\ref{fig:queryBounds}
demonstrates how queries can be used to construct upper and lower
bounds on general \LP[p] costs. The following Lemma also summarizes
well known bounds on general \LP[p] costs based on an \LP[1] cost.

\begin{lemma}
  \label{lma:L1-bound}
  The largest \LP[p] ($p>1$) ball enclosed within an \LP[1] ball has a
  radius (cost) of $\dims^{\frac{1-p}{p}}$ and for $p=\infty$ the
  radius is $\dims^{-1}$.
\end{lemma}

\subsubsection{Bounding \LP[p] Balls}
\label{sec:boundingLPBalls}

\begin{figure}[t]
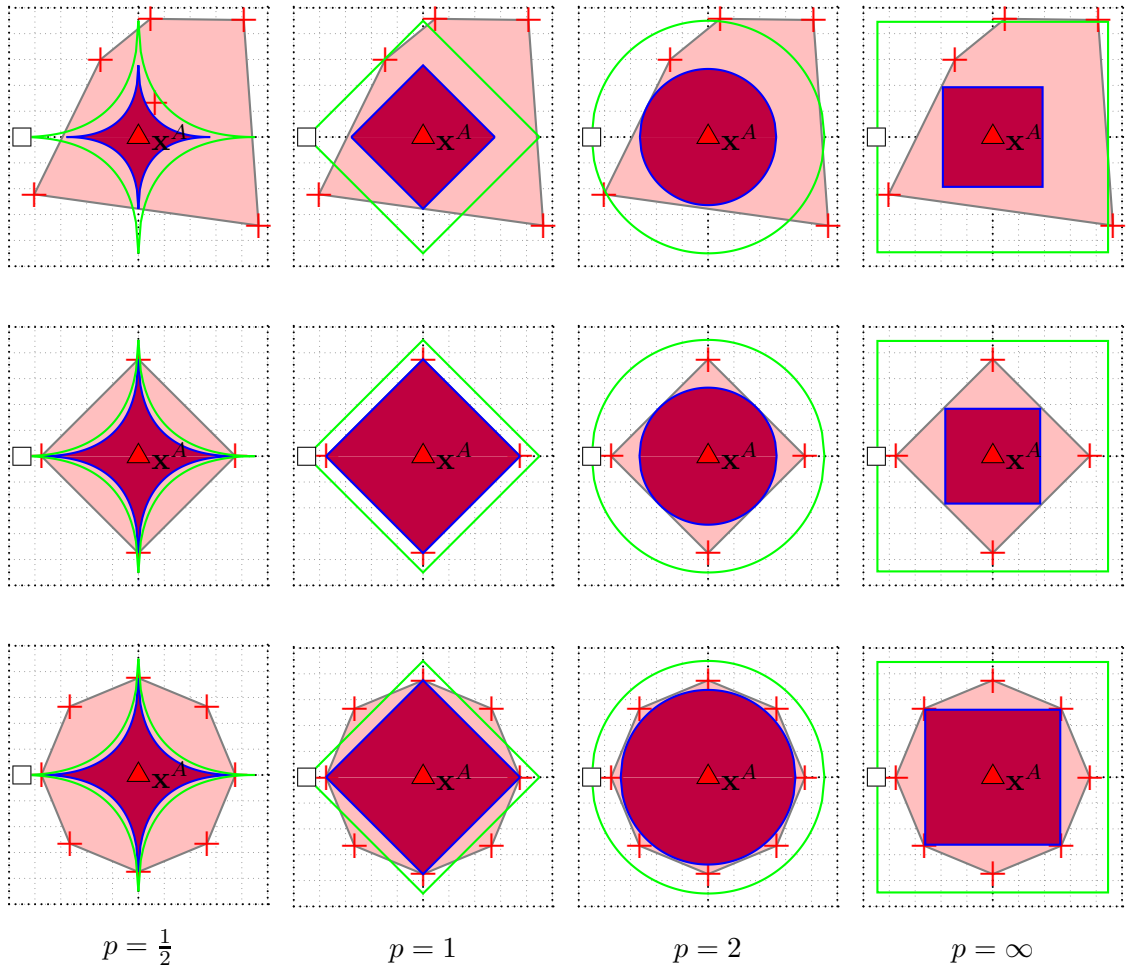


\begin{minipage}{.24\linewidth}
\begin{center}
\renewcommand{\picwidth}{0.47\linewidth}
\psset{unit=\picwidth}
\pspicture(-1,-1)(1,1)
  \psgrid[griddots=20,subgriddots=5,gridlabels=0pt](0,0)(-1,-1)(1,1)
  \savedata{\probes}[{{0.8147,0.9058},{0.1270,0.2647},{-0.8049,-0.4430},{0.0938,0.9150},{0.9298,-0.6848},{-0.292,0.6006}}]
  \savedata{\hull}[{{0.8147,0.9058},{0.0938,0.9150},{-0.292,0.6006},{-0.8049,-0.4430},{0.9298,-0.6848}}]
  \dataplot[fillstyle=solid,fillcolor=pink,plotstyle=polygon,linecolor=gray]{\hull}
  \dataplot[plotstyle=dots,showpoints=true,dotstyle=+,dotsize=8pt,linecolor=red]{\probes}
  \lpcontourplot[linecolor=blue,fillstyle=solid,fillcolor=purple]{0.5}{0.555195}
  \lpcontourplot[linecolor=green]{0.5}{0.9}
  \placeTarget
  \psdots[dotstyle=square,dotsize=8pt,linecolor=black](-0.9,0.0)
\endpspicture
\end{center}
\end{minipage}
\begin{minipage}{.24\linewidth}
\begin{center}
\renewcommand{\picwidth}{0.47\linewidth}
\psset{unit=\picwidth}
\pspicture(-1,-1)(1,1)
  \psgrid[griddots=20,subgriddots=5,gridlabels=0pt](0,0)(-1,-1)(1,1)
  \savedata{\probes}[{{0.8147,0.9058},{0.1270,0.2647},{-0.8049,-0.4430},{0.0938,0.9150},{0.9298,-0.6848},{-0.292,0.6006}}]
  \savedata{\hull}[{{0.8147,0.9058},{0.0938,0.9150},{-0.292,0.6006},{-0.8049,-0.4430},{0.9298,-0.6848}}]
  \dataplot[fillstyle=solid,fillcolor=pink,plotstyle=polygon,linecolor=gray]{\hull}
  \dataplot[plotstyle=dots,showpoints=true,dotstyle=+,dotsize=8pt,linecolor=red]{\probes}
  \lpcontourplot[linecolor=blue,fillstyle=solid,fillcolor=purple]{1.0}{0.555195}
  \lpcontourplot[linecolor=green]{1.0}{0.9}
  \placeTarget
  \psdots[dotstyle=square,dotsize=8pt,linecolor=black](-0.9,0.0)
\endpspicture
\end{center}
\end{minipage}
\begin{minipage}{.24\linewidth}
\begin{center}
\renewcommand{\picwidth}{0.47\linewidth}
\psset{unit=\picwidth}
\pspicture(-1,-1)(1,1)
  \psgrid[griddots=20,subgriddots=5,gridlabels=0pt](0,0)(-1,-1)(1,1)
  \savedata{\probes}[{{0.8147,0.9058},{0.1270,0.2647},{-0.8049,-0.4430},{0.0938,0.9150},{0.9298,-0.6848},{-0.292,0.6006}}]
  \savedata{\hull}[{{0.8147,0.9058},{0.0938,0.9150},{-0.292,0.6006},{-0.8049,-0.4430},{0.9298,-0.6848}}]
  \dataplot[fillstyle=solid,fillcolor=pink,plotstyle=polygon,linecolor=gray]{\hull}
  \dataplot[plotstyle=dots,showpoints=true,dotstyle=+,dotsize=8pt,linecolor=red]{\probes}
  \lpcontourplot[linecolor=blue,fillstyle=solid,fillcolor=purple]{2.0}{0.526973}
  \lpcontourplot[linecolor=green]{2.0}{0.9}
  \placeTarget
  \psdots[dotstyle=square,dotsize=8pt,linecolor=black](-0.9,0.0)
\endpspicture
\end{center}
\end{minipage}
\begin{minipage}{.24\linewidth}
\begin{center}
\renewcommand{\picwidth}{0.47\linewidth}
\psset{unit=\picwidth}
\pspicture(-1,-1)(1,1)
  \psgrid[griddots=20,subgriddots=5,gridlabels=0pt](0,0)(-1,-1)(1,1)
  \savedata{\probes}[{{0.8147,0.9058},{0.1270,0.2647},{-0.8049,-0.4430},{0.0938,0.9150},{0.9298,-0.6848},{-0.292,0.6006}}]
  \savedata{\hull}[{{0.8147,0.9058},{0.0938,0.9150},{-0.292,0.6006},{-0.8049,-0.4430},{0.9298,-0.6848}}]
  \dataplot[fillstyle=solid,fillcolor=pink,plotstyle=polygon,linecolor=gray]{\hull}
  \dataplot[plotstyle=dots,showpoints=true,dotstyle=+,dotsize=8pt,linecolor=red]{\probes}
  \lpcontourplot[linecolor=blue,fillstyle=solid,fillcolor=purple]{infty}{0.39369}
  \lpcontourplot[linecolor=green]{infty}{0.9}
  \placeTarget
  \psdots[dotstyle=square,dotsize=8pt,linecolor=black](-0.9,0.0)
\endpspicture
\end{center}
\end{minipage}
\vspace{2em}

\begin{minipage}{.24\linewidth}
\begin{center}
\renewcommand{\picwidth}{0.47\linewidth}
\psset{unit=\picwidth}
\pspicture(-1,-1)(1,1)
  \psgrid[griddots=20,subgriddots=5,gridlabels=0pt](0,0)(-1,-1)(1,1)
  \circleQuery[fillstyle=solid,fillcolor=pink,plotstyle=polygon,linecolor=gray]{4}{.75}
  \circleQuery[plotstyle=dots,showpoints=true,dotstyle=+,dotsize=8pt,linecolor=red]{4}{.75}
  \lpcontourplot[linecolor=blue,fillstyle=solid,fillcolor=purple]{0.5}{0.75}
  \lpcontourplot[linecolor=green]{0.5}{0.9}
  \placeTarget
  \psdots[dotstyle=square,dotsize=8pt,linecolor=black](-0.9,0.0)
\endpspicture
\end{center}
\end{minipage}
\begin{minipage}{.24\linewidth}
\begin{center}
\renewcommand{\picwidth}{0.47\linewidth}
\psset{unit=\picwidth}
\pspicture(-1,-1)(1,1)
  \psgrid[griddots=20,subgriddots=5,gridlabels=0pt](0,0)(-1,-1)(1,1)
  \circleQuery[fillstyle=solid,fillcolor=pink,plotstyle=polygon,linecolor=gray]{4}{.75}
  \circleQuery[plotstyle=dots,showpoints=true,dotstyle=+,dotsize=8pt,linecolor=red]{4}{.75}
  \lpcontourplot[linecolor=blue,fillstyle=solid,fillcolor=purple]{1.0}{0.75}
  \lpcontourplot[linecolor=green]{1.0}{0.9}
  \placeTarget
  \psdots[dotstyle=square,dotsize=8pt,linecolor=black](-0.9,0.0)
\endpspicture
\end{center}
\end{minipage}
\begin{minipage}{.24\linewidth}
\begin{center}
\renewcommand{\picwidth}{0.47\linewidth}
\psset{unit=\picwidth}
\pspicture(-1,-1)(1,1)
  \psgrid[griddots=20,subgriddots=5,gridlabels=0pt](0,0)(-1,-1)(1,1)
  \circleQuery[fillstyle=solid,fillcolor=pink,plotstyle=polygon,linecolor=gray]{4}{.75}
  \circleQuery[plotstyle=dots,showpoints=true,dotstyle=+,dotsize=8pt,linecolor=red]{4}{.75}
  \lpcontourplot[linecolor=blue,fillstyle=solid,fillcolor=purple]{2.0}{0.53033}
  \lpcontourplot[linecolor=green]{2.0}{0.9}
  \placeTarget
  \psdots[dotstyle=square,dotsize=8pt,linecolor=black](-0.9,0.0)
\endpspicture
\end{center}
\end{minipage}
\begin{minipage}{.24\linewidth}
\begin{center}
\renewcommand{\picwidth}{0.47\linewidth}
\psset{unit=\picwidth}
\pspicture(-1,-1)(1,1)
  \psgrid[griddots=20,subgriddots=5,gridlabels=0pt](0,0)(-1,-1)(1,1)
  \circleQuery[fillstyle=solid,fillcolor=pink,plotstyle=polygon,linecolor=gray]{4}{.75}
  \circleQuery[plotstyle=dots,showpoints=true,dotstyle=+,dotsize=8pt,linecolor=red]{4}{.75}
  \lpcontourplot[linecolor=blue,fillstyle=solid,fillcolor=purple]{infty}{0.375}
  \lpcontourplot[linecolor=green]{infty}{0.9}
  \placeTarget
  \psdots[dotstyle=square,dotsize=8pt,linecolor=black](-0.9,0.0)
\endpspicture
\end{center}
\end{minipage}
\vspace{2em}

\begin{minipage}{.24\linewidth}
\begin{center}
\renewcommand{\picwidth}{0.47\linewidth}
\psset{unit=\picwidth}
\pspicture(-1,-1)(1,1)
  \psgrid[griddots=20,subgriddots=5,gridlabels=0pt](0,0)(-1,-1)(1,1)
  \circleQuery[fillstyle=solid,fillcolor=pink,plotstyle=polygon,linecolor=gray]{8}{.75}
  \circleQuery[plotstyle=dots,showpoints=true,dotstyle=+,dotsize=8pt,linecolor=red]{8}{.75}
  \lpcontourplot[linecolor=blue,fillstyle=solid,fillcolor=purple]{0.5}{0.75}
  \lpcontourplot[linecolor=green]{0.5}{0.9}
  \placeTarget
  \psdots[dotstyle=square,dotsize=8pt,linecolor=black](-0.9,0.0)
\endpspicture\\[0.5em]
$p = \frac{1}{2}$
\end{center}
\end{minipage}
\begin{minipage}{.24\linewidth}
\begin{center}
\renewcommand{\picwidth}{0.47\linewidth}
\psset{unit=\picwidth}
\pspicture(-1,-1)(1,1)
  \psgrid[griddots=20,subgriddots=5,gridlabels=0pt](0,0)(-1,-1)(1,1)
  \circleQuery[fillstyle=solid,fillcolor=pink,plotstyle=polygon,linecolor=gray]{8}{.75}
  \circleQuery[plotstyle=dots,showpoints=true,dotstyle=+,dotsize=8pt,linecolor=red]{8}{.75}
  \lpcontourplot[linecolor=blue,fillstyle=solid,fillcolor=purple]{1.0}{0.75}
  \lpcontourplot[linecolor=green]{1.0}{0.9}
  \placeTarget
  \psdots[dotstyle=square,dotsize=8pt,linecolor=black](-0.9,0.0)
\endpspicture\\[0.5em]
$p = 1$
\end{center}
\end{minipage}
\begin{minipage}{.24\linewidth}
\begin{center}
\renewcommand{\picwidth}{0.47\linewidth}
\psset{unit=\picwidth}
\pspicture(-1,-1)(1,1)
  \psgrid[griddots=20,subgriddots=5,gridlabels=0pt](0,0)(-1,-1)(1,1)
  \circleQuery[fillstyle=solid,fillcolor=pink,plotstyle=polygon,linecolor=gray]{8}{.75}
  \circleQuery[plotstyle=dots,showpoints=true,dotstyle=+,dotsize=8pt,linecolor=red]{8}{.75}
  \lpcontourplot[linecolor=blue,fillstyle=solid,fillcolor=purple]{2.0}{0.6757}
  \lpcontourplot[linecolor=green]{2.0}{0.9}
  \placeTarget
  \psdots[dotstyle=square,dotsize=8pt,linecolor=black](-0.9,0.0)
\endpspicture\\[0.5em]
$p = 2$
\end{center}
\end{minipage}
\begin{minipage}{.24\linewidth}
\begin{center}
\renewcommand{\picwidth}{0.47\linewidth}
\psset{unit=\picwidth}
\pspicture(-1,-1)(1,1)
  \psgrid[griddots=20,subgriddots=5,gridlabels=0pt](0,0)(-1,-1)(1,1)
  \circleQuery[fillstyle=solid,fillcolor=pink,plotstyle=polygon,linecolor=gray]{8}{.75}
  \circleQuery[plotstyle=dots,showpoints=true,dotstyle=+,dotsize=8pt,linecolor=red]{8}{.75}
  \lpcontourplot[linecolor=blue,fillstyle=solid,fillcolor=purple]{infty}{0.53033}
  \lpcontourplot[linecolor=green]{infty}{0.9}
  \placeTarget
  \psdots[dotstyle=square,dotsize=8pt,linecolor=black](-0.9,0.0)
\endpspicture\\[0.5em]
$p = \infty$
\end{center}
\end{minipage}
\caption{Convex hull for a set of queries and the resulting bounding
  balls for several \LP\ costs. Each row
  represents a unique set of positive (red \posLbl\ points) and
  negative (green \negLbl\ points) queries and each column shows the
  implied upper bound (in green) and lower bound (in blue) for a
  different \LP[p] cost. In the first row, the body is defined by a
  random set of $7$ queries, in the second, the queries are along the
  coordinate axes, and in the third, the queries are around a circle.}
\label{fig:queryBounds}
\end{figure}

In general, suppose we probe along some set of $M$ unit directions and
at some point we have at least one negative point supporting an upper
bound of \maxCost\ and $M$ positive points supporting at a cost of
\minCost. However, the lower bound provided by those $M$ positive
points is the cost of the largest \LP[p] cost ball that fits entirely
within their convex hull; let's say this cost is $C^\dagger <
\minCost$. In order to achieve \multGoal-multiplicative optimality, we
need
\[
  \frac{\maxCost}{C^\dagger} \le 1 + \multGoal \enspace.
\]
Expanding this, we need
\[
  \left(\frac{\maxCost}{\minCost}\right) \left(\frac{\minCost}{C^\dagger}\right) \le 1 + \multGoal \enspace.
\]
This allows us to break the problem into two parts. The first factor
$\maxCost/\minCost$ is only in terms of parameters controlled
by the multiline search algorithm whereas the second factor
$\minCost/C^\dagger$ depends only on the shape of the \LP[p]
ball as it captures how well the ball is approximated by the convex
hull of the search directions. These two factors separate our task
into choosing $M$ and \steps\ sufficiently so that their product is less
than $1 + \multGoal$. First we choose factors $\alpha \ge 0$ and
$\beta \ge 0$ so that $(1+\alpha)(1+\beta) \le 1+\multGoal$. Then we
chose $M$ so that
\[
  \frac{\minCost}{C^\dagger} = 1 + \beta
\]
and a parameter $\multGoal^\prime = \alpha$ so that multiline search
with $M$ directions will achieve
\[
  \frac{\maxCost}{\minCost} = 1 + \alpha \enspace .
\]
In doing so, we create a generalized multiline search that is able to
achieve \multGoal-multiplicative optimality.

For example in the case of $p = 1$, we previously saw that choosing $M
= 2\cdot\dims$ allows us to exactly reconstruct the \LP[1] ball so
that $\minCost/C^\dagger = 1$ (\ie\ $\beta = 0$). Thus we can
just make $\alpha = \multGoal$ and we recover our original multiline
search method exactly.

\paragraph{Objective:}

Below we present a number of results that deal with cases when $\beta
> 0$. In this case, what we want to show is that a ratio of
$\frac{\minCost}{C^\dagger} = 1 + \beta$ can be achieved with a
polynomial number of search directions when $\beta \le \multGoal$;
otherwise, $(1+\alpha)(1+\beta) > 1 + \epsilon$. Thus, we will be
trying to find how many search directions are required for to achieve
\[
  \frac{\minCost}{C^\dagger} \le 1 + \epsilon \enspace,
\]
since this is the highest we can allow this ratio to be. Moreover,
since this problem scales linearly with $\minCost$ we will simply
examine the values of $C^\dagger$ that can be achieved for the unit
cost ball (\ie\ w.l.o.g. we make $\minCost = 1$ and rescale). Thus we
will be looking at how many points are required to achieve:
\begin{equation}
  C^\dagger \ge \frac{1}{1+\multGoal} \label{eq:innerCostGoal} \enspace.
\end{equation}
We will try to show that only polynomially many are required for at
least some values of \multGoal.

\begin{lemma}
  \label{lma:enclosedBound}
  If there exists a configuration of $M$ unit search directions with a
  convex hull that yields a bound $C^\dagger$ for the cost function
  $\adCost$ then multi-line search algorithms can use those search
  directions to achieve \multGoal-multiplicative optimality with a
  query complexity that is polynomial in $M$ and \multSteps\ for any
  \[
    \multGoal > \frac{1}{C^\dagger} - 1 \enspace.
  \]
\end{lemma}

\begin{corollary}
  \label{cor:perfectBound}
  If there exists a configuration of $M$ unit search directions with a
  convex hull that yields a bound $C^\dagger = 1$ for the cost
  function $\adCost$ then multi-line search algorithms then multi-line
  search algorithms can use those search directions to achieve
  \multGoal-multiplicative optimality with a query complexity that is
  polynomial in $M$ and \multSteps\ for any $\multGoal > 0$.
\end{corollary}

As this corollary reaffirms, for $p=1$ using the $M = 2 \cdot \dims$
coordinate directions allows multi-line search algorithms to achieve
\multGoal-multiplicative optimality for any $\multGoal > 0$ with a
query complexity that is polynomial in $M$ and \multSteps.

\subsubsection{Multiline Search for $p<1$}

A simple result holds here. Namely, since the unit \LP[1] ball bounds
any unit \LP[p] balls with $p < 1$ we can achieve
$\minCost/C^\dagger = 1$ using only the $2\cdot\dims$ corners
of the hyperoctahedron as search directions. Thus we can efficiently
search for $p < 1$ for any value of $\multGoal > 0$. Whether or not
the \LP[p] ($p < 1$) cost functions can be efficiently searched with
fewer search directions is an open question.

\subsubsection{Multiline Search for $p>1$}

For this case, we can trivially use the \LP[1] bound on \LP[p] balls as summarized by the following corollary:
\begin{corollary}
  \label{thm:LP-polynomial}
  For $1 < p < \infty$ and $\multGoal \in
  \left(\dims^{\frac{p-1}{p}}-1,\infty\right)$ any multi-line search
  algorithm can achieve \multGoal-multiplicative optimality on
  $\lpCost{p}$ using $M = 2\cdot\dims$ search directions. Similarly
  for $p=\infty$ and $\multGoal \in \left(\dims-1,\infty\right)$ any
  multi-line search algorithm can achieve \multGoal-multiplicative
  optimality on $\lpCost{\infty}$.
\end{corollary}
\begin{proof}
  From Lemma~\ref{lma:L1-bound}, the largest co-centered \LP[p]
  ball contained within the unit \LP[1] ball has radius (cost)
  $\dims^{\frac{1-p}{p}}$ (or $\dims$ for $p=\infty$). The bounds on
  $\multGoal$ then follows from Lemma~\ref{lma:enclosedBound}.
\end{proof}

Unfortunately, this result only applies for a range of \multGoal\ that
grows with \dims, which is insufficient for \kSearchability. In fact,
for some fixed values of \multGoal, there is no query-based strategy
that can bound \LP[p] costs using polynomially-many queries in \dims
as the following result formalizes.

\begin{theorem}
  \label{thm:LP-exponential}
  For $p > 1$, $\dims>0$, any initial bounds $0 < \minCost < \maxCost$
  on the \MAC, and $0 < \multGoal < 2^{\frac{p-1}{p}} - 1$ (or $0 <
  \multGoal < 1$ for $p=\infty$), all algorithms must submit at least
  $\alpha_{p,\multGoal}^\dims$ membership queries (for some constant
  $\alpha_{p,\multGoal} > 1$) in the worst case to be
  \multGoal-multiplicatively optimal on
  $\classSpace^\mathrm{convex,\posLbl}$ for \LP[p] costs.
\end{theorem}

The proof of this theorem is in Appendix~\ref{sec:LP-exponential}.  A
consequence of this theorem is that there is no query-based algorithm
that can efficiently find an \kIMAC\ of any \LP[p] cost ($p>1$) for
any $0 < \multGoal < 2^{\frac{p-1}{p}}$ (or $0 < \multGoal < 1$ for
$p=\infty$) on the family \familyConvexPos. However, from
Theorem~\ref{thm:LP-polynomial} and Lemma~\ref{lma:enclosedBound},
multiline-search type algorithms efficiently find the \kIMAC\ of any
\LP[p] cost ($p>1$) for any $\multGoal \in
\left(\dims^{\frac{p-1}{p}}-1,\infty\right)$ (or $\dims-1 < \multGoal
< \infty$ for $p=\infty$). It is generally unclear if efficient
algorithms exist for any values of \multGoal\ between these intervals,
but in the following section we derive a stronger bound for the case
of $p=2$.

\subsubsection{Multiline Search for $p=2$}
\label{sec:MLS-L2}

\begin{theorem}
  \label{thm:L2-exponential}
  For any $\dims>1$, any initial bounds $0 < \minCost < \maxCost$ on
  the \MAC, and $0 < \multGoal < \frac{\maxCost}{\minCost} - 1$, all
  algorithms must submit at least
  $\alpha_\multGoal^{\frac{\dims-2}{2}}$ membership queries (where
  $\alpha_\multGoal = \frac{(1+\multGoal)^2}{(1+\multGoal)^2-1} > 1$)
  in the worst case to be \multGoal-multiplicatively optimal on
  $\classSpace^\mathrm{convex,\posLbl}$ for \LP[2] costs.
\end{theorem}
The proof of this result is in Appendix~\ref{sec:L2-exponential}.

This result says that no algorithm can achieve
\multGoal-multiplicative optimality for \LP[2] costs for any fixed
$\multGoal > 0$ using only polynomially-many queries in
$\dims$. However, for a fixed $\dims$, the bound provided by
Theorem~\ref{thm:L2-exponential} suggests that reasonable
approximations may be achievable $\alpha_\multGoal \to 1$.

It may appear that Theorem~\ref{thm:L2-exponential} contradicts
Corollary~\ref{thm:LP-polynomial}. However, in
Corollary~\ref{thm:LP-polynomial} only applies for a range of
$\multGoal$ that depends on $\dims$; \ie\ $\multGoal >
\sqrt{\dims}-1$. Interestingly, substituting this lower bound on
$\multGoal$ into the bound given by Theorem~\ref{thm:L2-exponential},
we get that the number of required queries for $\multGoal >
\sqrt{\dims}-1$ need only be
\begin{equation*}
  M \quad \ge \quad \left(\frac{(1+\multGoal)^2}{(1+\multGoal)^2-1}\right)^{\frac{\dims-2}{2}}
  \quad = \quad \left(\frac{\dims}{\dims-1}\right)^{\frac{\dims-2}{2}}
\end{equation*}
which is a monotonically increasing function in $\dims$ that
asymptotes at $\sqrt{e} \approx 1.64$. Thus,
Theorem~\ref{thm:L2-exponential} and Corollary~\ref{thm:LP-polynomial}
are in agreement since for $\multGoal > \sqrt{\dims}-1$, the former
only requires that we need at least $2$ queries.

\subsection{Convex Negative Set}
\label{sec:genLPNegative}

Algorithm~\ref{alg:setsearch} generalizes immediately to all weighted
\LP[p] costs ($p \ge 1$) centered at $\xtarget$ since these costs are
convex.  For these costs an equivalent separating hyperplane for
$\vec{y}$ can be used in place of Eq.~\eqref{eq:gradient}. These are
given by the equivalent (sub)-gradients for \LP[p] cost-balls:
\begin{eqnarray*}
  h_{p,d}^\vec{y} & = & \ithCost{d} \sign\left(y_d - \xtargetc_d\right) \cdot \left(\frac{|y_d - \xtargetc_d|}{\lpCostFunc[\featCost]{p}{\vec{y}}}\right)^{p-1} \\
  h_{\infty,d}^\vec{y} & = & \ithCost{d} \sign\left(y_d - \xtargetc_d\right) \cdot \ind{|y_d - \xtargetc_d| = \lpCostFunc[\featCost]{p}{\vec{y}}} \enspace.
\end{eqnarray*}
By only changing the cost function $\adCost$ and the separating
hyperplane $\vec{h}^\vec{y}$ used for the halfspace cut in
Algorithms~\ref{alg:intersect} and~\ref{alg:setsearch}, the randomize
ellipsoid search can be applied for any weighted \LP[p] cost
\lpCost[\featCost]{p}.

For more general convex
costs $\adCost$, we still have that the set of all points $\dpt$ with
$\adCostFunc{\dpt} \le C$ (the \term{sublevel set} of cost $C$) is a
subset of the sublevel set of cost $D$ for all $D > C$; thus, the
separating hyperplanes for the sublevel set at cost $D$ will also be
separating hyperplanes for the sublevel set at cost $C$. The
\algoName{SetSearch} procedure therefore is applicable for any convex
cost function $\adCost$ so long as we can compute the separating
hyperplanes of any sublevel set of $\adCost$ for any point $\vec{y}$
not in sublevel set\footnote{The sublevel set of any convex function
  is a convex set \citep[see][]{ConvexOptimization}
  so such a separating hyperplane always exists but may not be simple
  to compute.}.

For non-convex costs \adCost\ such as weighted \LP[p] costs with $p <
1$, minimizing on a convex set \xminus\ is generally a hard problem.
However, there may be special cases when minimizing such a cost can be
accomplished efficiently.
\section{Conclusions and Future Work}
\label{sec:conclusion}

In this paper we study \kSearchability\ of \convexClass. We
present membership query algorithms that efficiently accomplish
\kIMAC\ search on this family.  When the positive class is
convex we demonstrate very efficient techniques that outperform the
previous reverse-engineering approaches for linear classifiers. When
the negative class is convex, we apply a randomized Ellipsoid method
to achieve efficient \kIMAC\ search. If the adversary is unaware
of which set is convex, they can trivially run both searches to
discover an \kIMAC\ with a combined polynomial query complexity.
We also show our algorithms can be efficiently extended to
cope with a number of special circumstances.
Most importantly, we demonstrate that
these algorithms can succeed without reverse engineering the
classifier. Instead, these algorithms systematically eliminate
inconsistent hypotheses and progressively concentrate their efforts in
an ever-shrinking neighborhood of a \MAC\ instance. By doing so, these
algorithms only require polynomially-many queries in spite of the size
of the family of all \convexClass.

We also consider general \LP\ costs and show that
$\classSpace^{convex}$ is only \kSearchable\ for both positive and
negative convexity for any $\multGoal>0$ if $p=1$. For $0 < p < 1$,
the \MLS\ algorithms of Section~\ref{sec:pos-convex} achieve identical
results when the positive set is convex, but the non-convexity of
these \LP\ costs precludes the use of our randomized Ellipsoid method.
The Ellipsoid method does provide an efficient solution for convex
negative sets when $p>1$ (since these costs are convex). However, for
convex positive sets, our results show that for $p>1$ there is no
algorithm that can efficiently find an \kIMAC\ for all $\multGoal >
0$. Moreover, for $p=2$ we prove that there is no efficient algorithm
for finding an \kIMAC\ for any fixed value of \multGoal.

By studying \kIMAC\ searchability, 
we provide a broader picture of how machine learning techniques are 
vulnerable to query-based evasion attacks.
Exploring near-optimal evasion is important for understanding how an
adversary may circumvent learners in security-sensitive settings. In
such an environment, system developers are hesitant to trust
procedures that may create vulnerabilities. 
The algorithms we
demonstrate are invaluable tools not for an adversary to develop
better attacks but rather for analysts to better understand the
vulnerabilities of their filters.  Our algorithms may not necessarily
be easily used by an adversary since
various real-world obstacles would first need to be overcome. Queries
may only be partially observable or noisy and
the feature set may only be partially known. Moreover,
an adversary may not be able to query all $\dpt \in
\xspace$; instead their queries must be legitimate objects (such as email) that
are mapped into $\xspace$. A real-world adversary must invert the
feature-mapping---a generally difficult task. These limitations necessitate
further research on the impact of partial observability and
approximate querying on \kIMAC\ search, and to design more secure filters.
Broader open problems include: is \kIMAC\ search possible on
other classes of learners such as SVMs (linear in a large possibly
infinite feature space)? Is \kIMAC\ search feasible against an
online learner that adapts as it is queried?
Can learners be made resilient to these threats and how does this
impact learning performance?
\section*{Acknowledgements}

We would like to thank Peter Bartlett, Marius Kloft, and Peter Bodik
for their helpful feedback on this project.

We gratefully acknowledge the support of our sponsors. This work was
supported in part by TRUST (Team for Research in Ubiquitous Secure
Technology), which receives support from the National Science
Foundation (NSF award \#CCF-0424422) and AFOSR (\#FA9550-06-1-0244); RAD
Lab, which receives support from California state MICRO grants
(\#06-148 and \#07-012); DETERlab (cyber-DEfense Technology Experimental
Research laboratory), which receives support from DHS HSARPA (\#022412)
and AFOSR (\#FA9550-07-1-0501); NSF award \#DMS-0707060; the Siebel Scholars
Foundation; and the
following organizations: Amazon, BT, Cisco, DoCoMo USA Labs, EADS,
ESCHER, Facebook, Google, HP, IBM, iCAST, Intel, Microsoft, NetApp,
ORNL, Pirelli, Qualcomm, Sun, Symantec, TCS, Telecom Italia, United
Technologies, and VMware. The opinions expressed in this paper are
solely those of the authors and do not necessarily reflect the
opinions of any funding agency, the State of California, or the
U.S. government.

\bibliography{sources}

\appendix

\section{Proof of Theorems for \MLS\ Algorithms}
\label{app:MLS-proofs}

To analyze the worst case of \KMLS\ (Algorithm~\ref{alg:kmls}), we
consider a \term{\malcal} that maximizes the number of queries. We
refer to the agent that queries the classifier as the
\textit{adversary}.

\begin{proofof}{Theorem~\ref{thm:sqrtL}}
  At each each iteration of Algorithm~\ref{alg:kmls}, the adversary
  choses some direction, $\vec{e}$ not yet eliminated from $\aset{W}$.
  Every direction in $\aset{W}$ is feasible (\ie\ could yield an
  \kIMAC) and the \malcal, by definition, will make this
  choice as costly as possible.  During the $K$ steps of binary search
  along this direction, regardless of which direction $\vec{e}$ is
  selected or how the \malcal\ responds, the candidate multiplicative
  gap (see Section~\ref{sec:binSearch}) along $\vec{e}$ will shrink by
  an exponent of $2^{-K}$; \ie
  \begin{eqnarray}
    \frac{B^-}{B^+} & = & \left(\frac{C^-}{C^+}\right)^{2^{-K}} \\
    \log(G_{t+1}^\prime) & = & \log(G_t) \cdot 2^{-K} \label{eq:gap-shrink}
  \end{eqnarray}
  The primary decision for the \malcal\ occurs when the
  adversary begins querying other directions beside $\vec{e}$. At
  iteration $t$, the \malcal\ has 2 options:
  \begin{verse}
  Case 1 ($t \in \aset{C}_1$): Respond with \posLbl\ for all
    remaining directions. Here the bounds candidates $B^+$ and $B^-$
    are verified and thus the new gap
    is reduced by an exponent of $2^{-K}$; however, no directions are
    eliminated from the search.
  \end{verse}
  \begin{verse}
    Case 2 ($t \in \aset{C}_2$): Choose at least 1 direction to
    respond with \negLbl. Here since only the value of $C^-$ changes,
    the \malcal\ can chose to respond to the first $K$ queries so that
    the gap decreases by a neglibile amount (by always responding with
    \posLbl\ during the first $K$ queries along $\vec{e}$, the gap
    only decreases by an exponent of $(1-2^{-K})$).  However, the
    \malcal\ must chose some number $E_t \ge 1$ of directions that
    will be eliminated.
  \end{verse}
  We conservatively assume that the gap only decreases for case 1,
  which decouples the analysis of the queries for $\aset{C}_1$ and
  $\aset{C}_2$ and allows us to upper bound the total number of
  queries made by the algorithm.  By this assumption, if $t \in
  \aset{C}_1$ we have $G_t = G_{t-1}^{2^{-K}}$ whereas if $t \in
  \aset{C}_2$, we have $G_t = G_{t-1}$. By analyzing the gap before
  and after the final iteration $T$, it can be shown that
  \begin{equation}
    \label{eq:gap-tradeoff}
    |\aset{C}_1| = \left\lceil\frac{\steps}{K}\right\rceil
  \end{equation}
  since, for the algorithm to terminate, there must be a total of at
  least $\steps$ binary search steps made during the case 1 iterations
  and each case 1 iteration takes exactly $K$ steps.

  At every case 1 iteration, the adversary make exactly $K + |\aset{W}_t| - 1$
  queries where $\aset{W}_t$ is the set of feasible directions
  remaining at the \nth{t}{th} iteration. While $\aset{W}_t$ is
  controlled by the \malcal, we can apply the bound
  $|\aset{W}_t| \le |\aset{W}|$. Using this and the relation from
  Eq.~\eqref{eq:gap-tradeoff}, we can bound the number of queries
  $Q_1$ used in case 1 by
  \begin{eqnarray*}
    Q_1 & \le & \sum_{t \in C_1}{(K + |\aset{W}| - 1)} \\
        & = &  \left\lceil\frac{L}{K}\right\rceil \cdot \left(K + |\aset{W}|-1 \right) \\
        & \le & \left(\frac{L}{K}+1\right) \cdot K + \left\lceil\frac{L}{K}\right\rceil \cdot (|\aset{W}|-1) \\
        & = & L + K + \left\lceil\frac{L}{K}\right\rceil \cdot \left(|\aset{W}|-1\right) \enspace.
  \end{eqnarray*}
  
  For each case 2 iteration, we make exactly $K + E_t$ queries and
  this causes the elimination of $E_t \ge 1$ directions; hence,
  $|\aset{W}_{t+1}| = |\aset{W}_t| - E_t$.  A \malcal\ will always make
  $E_t=1$ whenever they use case 2 since that maximally limits how
  much the adversary gains. Nevertheless, since case 2 requires the
  elimination of at least 1 direction, we have $|\aset{C}_2| \le
  |\aset{W}| - 1$ and moreover, regardless of the choice of $E_t$ we
  have $\sum_{t \in \aset{C}_2}{E_t} \le |\aset{W}| - 1$ since each
  direction can be eliminated no more than once. Thus,
  \begin{eqnarray*}
    Q_2 & = & \sum_{i \in \aset{C}_2}{(K + E_t)} \\
        & \le & |\aset{C}_2| \cdot K + |\aset{W}|-1 \\
        & \le & \left(|\aset{W}|-1\right)\left(K + 1\right)
        \enspace .
  \end{eqnarray*}

  The total number of queries used by Algorithm~\ref{alg:kmls}
  \begin{eqnarray*}
    Q = Q_1+Q_2 & \le & L + K + \left\lceil\frac{L}{K}\right\rceil \cdot \left(|\aset{W}|-1\right) + \left(|\aset{W}|-1\right)\left(K + 1\right) \\
      & = & L + \left\lceil\frac{L}{K}\right\rceil \cdot |\aset{W}| + K\cdot|\aset{W}| + |\aset{W}| - \left\lceil\frac{L}{K}\right\rceil - 1 \\
      & = & L + \left(\left\lceil\frac{L}{K}\right\rceil + K + 1\right) |\aset{W}| \enspace.
  \end{eqnarray*}

  Finally, choosing $K=\lceil \sqrt{L} \rceil$ minimizes this
  expression and using $L/\lceil\sqrt{L}\rceil \le \sqrt{L}$ and
  substituting $K$ into $Q$'s bound, we have
  \[
    Q \le L + \left(2 \lceil\sqrt{L}\rceil + 1\right) |\aset{W}| \enspace.
  \]
\end{proofof}

\section{Proof of Lower Bounds}
\label{app:lowerBounds}

Here we give proofs for the lower bound theorems in
Section~\ref{sec:lowBounds} first giving the proof for the more
complictated multiplicative case followed by a similar proof sketch
for the additive case. For these lower bounds, $\dims$ is the
dimension of the space, $\adCost : \reals^\dims \to \realpos$ is
any positive convex function, $0 < \minCost < \maxCost$ are initial
upper and lower bounds on the \MAC, and
$\hat{\classSpace}^\mathrm{convex,\posLbl} \subset
\classSpace^\mathrm{convex,\posLbl}$ is the set of classifiers
consistent with the constraints on the \MAC; \ie\ for $\classifier
\in \hat{\classSpace}^\mathrm{convex,\posLbl}$ we have \xplus is
convex, $\ball{\minCost}{\adCost} \subset \xplus$, and
$\ball{\maxCost}{\adCost} \not\subset \xplus$.

\begin{proofof}{Theorems~\ref{thm:alower} and~\ref{thm:mlower}}
  Suppose a query-based algorithm submits $N < \dims + 1$ membership
  queries $\dpt^{1},\ldots,\dpt^{N}\in\reals^{\dims}$ to the
  classifier. For the algorithm to be \multGoal-optimal, these queries
  must constrain all consistent classifiers
  $\hat{\classSpace}^\mathrm{convex,\posLbl}$ to have a common point
  among their \kIMAC[\multGoal] sets.
  Suppose that the responses to the queries are consistent with the
  classifier $\classifier$ defined as:
  \begin{equation}
  \prediction{\dpt} = \begin{cases}
    +1\:, & \mbox{if } \adCostFunc{\dpt} < \maxCost \\
    -1\:, & \mbox{otherwise}
  \end{cases}\enspace.
  \label{eq:ballPredictor}
  \end{equation}
  For this classifier, $\xplus$ is convex since $\adCost$ is a convex
  function, $\ball{\minCost}{\adCost} \subset \xplus$ since
  $\minCost < \maxCost$, and $\ball{\maxCost}{\adCost} \not\subset
  \xplus$ since $\xplus$ is the open \maxCost-ball whereas
  $\ball{\maxCost}{\adCost}$ is the closed \maxCost-ball.
  Moreover, since $\xplus$ is the open \maxCost-ball,
  $\nexists \; \dpt \in \xminus \; \mathrm{s.t.} \; \adCostFunc{\dpt}
  < \maxCost$ therefore
  $\function{\MAC}{\classifier,\adCost}=\maxCost$, and any
  \multGoal-optimal points $\dpt^\prime \in
  \function{\multIMAC}{\classifier,\adCost}$ must
  satisfy $\maxCost \le \adCostFunc{\dpt^\prime} \le (1 + \multGoal)
  \maxCost$. Similarly, any \addGoal-optimal points $\dpt^\prime \in
  \function{\addIMAC}{\classifier,\adCost}$ must satisfy
  $\maxCost \le \adCostFunc{\dpt^\prime} \le \maxCost + \addGoal$.

  Consider an alternative classifier $\funcName{g}$ that responds
  identically to $\classifier$ for $\dpt^{1},\ldots,\dpt^{N}$ but has
  a different convex positive set $\xplus[\funcName{g}]$. Without loss
  of generality, suppose the first $M \le N$ queries are positive and
  the remaining are negative. Let $\aset{G} =
  \function{conv}{\dpt^1,\ldots,\dpt^M}$; that is, the convex hull of
  the $M$ positive queries. Now let $\xplus[\funcName{g}]$ be the
  convex hull of $\aset{G}$ and the $\minCost$-ball of $\adCost$:
  $\xplus[\funcName{g}] = \function{conv}{\aset{G} \cup
    \ball{\minCost}{\adCost}}$.  Since $\aset{G}$ contains all
  positive queries and $\minCost < \maxCost$, the convex set
  $\xplus[\funcName{g}]$ is consistent with the observed responses,
  $\ball{\minCost}{\adCost} \subset \xplus[\funcName{g}]$ by
  definition, and $\ball{\maxCost}{\adCost} \not\subset
  \xplus[\funcName{g}]$ since the positive queries are all inside the
  open \maxCost-sublevel set. Further, since $M \le N < \dims + 1$,
  $\aset{G}$ is contained in a proper linear subspace of
  $\reals^\dims$ and hence $\interior{\aset{G}} = \emptyset$. Hence,
  there is always some point from $\ball{\minCost}{\adCost}$ that is
  on the boundary of $\xplus[\funcName{g}]$; \ie\ 
  $\ball{\minCost}{\adCost} \not\subset \interior{\aset{G}}$ because
  $\interior{\aset{G}} = \emptyset$ and
    $\ball{\minCost}{\adCost} \neq \emptyset$. Hence, there must be
  at least one point from $\ball{\minCost}{\adCost}$ on the boundary
  of the convex hull of $\ball{\minCost}{\adCost}$ and $\aset{G}$. 
  Hence, $\function{\MAC}{\funcName{g},\adCost}
  = \inf_{\dpt \in \xminus[\funcName{g}]}\left[ \adCostFunc{\dpt}
  \right] = \minCost$.  Since the accuracy $\multGoal <
  \frac{\maxCost}{\minCost} - 1$, any $\dpt \in
  \function{\multIMAC}{\funcName{g},\adCost}$ must have
  \[
    \adCostFunc{\dpt} \le (1+\multGoal)\minCost <
      \frac{\maxCost}{\minCost} \minCost = \maxCost \enspace, 
  \]
  whereas any $\vec{y} \in \function{\multIMAC}{\classifier,\adCost}$
  must have $\adCostFunc{\vec{y}} \ge \maxCost$.  Thus,
  $\function{\multIMAC}{\classifier,\adCost} \cap
  \function{\multIMAC}{\funcName{g},\adCost} = \emptyset$ and we
  have constructed two \convexClass\ $\classifier$ and $\funcName{g}$
  both consistent with the query responses with no common $\multIMAC$. Similarly, since $\addGoal < \maxCost - \minCost$, any $\dpt \in
  \function{\addIMAC}{\funcName{g},\adCost}$ must have
  \[
    \adCostFunc{\dpt} \le \addGoal + \minCost 
    < \maxCost - \minCost + \minCost = \maxCost \enspace,
  \]
  whereas any $\vec{y} \in \function{\addIMAC}{\classifier,\adCost}$
  must have $\adCostFunc{\vec{y}} \ge \maxCost$.  Thus,
  $\function{\addIMAC}{\classifier,\adCost} \cap
  \function{\addIMAC}{\funcName{g},\adCost} = \emptyset$ and so the two
  \convexClass\ $\classifier$ and $\funcName{g}$ also have no common
  $\addIMAC$.

  Suppose instead that a query-based algorithm submits $N <
  \multSteps$ membership queries (or $N < \addSteps$ for the additive case). 
  Recall our definitions: $\maxCost$ is the initial upper bound on the
  \MAC, $\minCost$ is the initial lower bound on the \MAC, and
  $G_t^{(\ast)} = \maxCost[t] / \minCost[t]$ is the gap between the
  upper bound and lower bound at iteration $t$ ($G_t^{(+)} =
  \maxCost[t] - \minCost[t]$ for the additive case). Here, the
  \malcal\ $\classifier$ responds with
  \begin{equation}
    \prediction{\dpt^t} = \begin{cases}
      +1\:, & \mbox{if }\adCostFunc{\dpt^t} \le \sqrt{\maxCost[t-1]\cdot\minCost[t-1]}\\
      -1\:, & \mbox{otherwise}\end{cases}\enspace.
    \label{eq:halvingPredictor}
  \end{equation}
  When the classifier responds with \posLbl, $\minCost[t]$ increases
  to no more than $\sqrt{\maxCost[t-1]\cdot\minCost[t-1]}$ and so $G_t
  \ge \sqrt{G_{t-1}}$. Similarly when this classifier responds with
  \negLbl, $\maxCost[t]$ decreases to no less than
  $\sqrt{\maxCost[t-1]\cdot\minCost[t-1]}$ and so again $G_t \ge
  \sqrt{G_{t-1}}$.  Thus, these responses ensure that at each
  iteration $G_t \ge \sqrt{G_{t-1}}$ and since the algorithm can not
  terminate until $G_N \le 1+\multGoal$, we have $N \ge \multSteps$
  from Eq.~\eqref{eq:Lmult} (or in the additive case $N \ge \addSteps$
  from Eq.~\ref{eq:Ladd}). Again we have constructed two
  \convexClass\ with consistent query responses but with no common
  \kIMAC. The first classifier's positive set is the smallest
  cost-ball enclosing all positive queries, while the second
  classifier's positive set is the largest cost-ball enclosing all
  positive queries but no negatives. The \MAC\ values of these sets
  differ by more than a factor of $(1+\multGoal)$ if $N < \multSteps$
  (or, for the additive case, by a difference of more than \addGoal\
  if $N < \addSteps$), so they have no common \kIMAC.
\end{proofof}

\section{Proof of Theorem~\ref{thm:LP-exponential}}
\label{sec:LP-exponential}

First we introduce the following lemma for the $\dims$-dimensional
\term{hypercube graphs}---a collection of $2^\dims$ nodes of the form
$(\pm 1, \pm 1, \ldots, \pm 1)$ where each node has an edge to every
other node that is Hamming distance $1$ from it.

\begin{lemma}
  \label{lma:hypercubeCover}
  For any $0 < \delta < 1/2$, to cover a $\dims$-dimensional hypercube
  graph so that every vertex has a Hamming distance of at most
  $\lfloor\delta\dims\rfloor$ to some vertex in the covering, the
  number of vertices in the covering must be
  \[
    \function{Q}{\dims,h} \ge 2^{\dims\left(1-\entropy{\delta}\right)} \enspace,
  \]
  where $\entropy{\delta} = -\delta \log_2 \delta -
  (1-\delta)\log(1-\delta)$ is the \term{entropy} of $\delta$.
\end{lemma}
\begin{proof}
  There are $2^\dims$ vertices in the $\dims$-dimensional hypercube
  graph. Each vertex in the covering is within a Hamming distance of
  at most $h$ for exactly $\sum_{k=0}^{h}{\binom{\dims}{k}}$
  vertices. Thus, one needs at least
  $2^{\dims}/\left(\sum_{k=0}^{h}{\binom{\dims}{k}}\right)$ to cover the
  hypercube graph.  Now we apply the bound
  \[
  \sum_{k=0}^{\lfloor\delta\dims\rfloor}{\binom{\dims}{k}} \le
  2^{\entropy{\delta}\dims}
  \]
  to the denominator, which is valid for any $0 < \delta < 1/2$.
\end{proof}

\newcommand{\halfspace}[2]{\ensuremath{\aset{H}_{#1,#2}}}
\newcommand{\lagrangian}[1]{\function{\mathcal{L}}{#1}}
\newcommand{\displacement}{\ensuremath{d}}

\begin{lemma}
  \label{thm:hyperplane-lp-minimizer}
  The minimizer of the \LP[p] cost function \lpCost{p} to any target
  \xtarget\ on the halfspace $\halfspace{\vec{w}}{\vec{b}} =
  \set[\dpt^\top\vec{w} \ge \vec{b}^\top\vec{w}]{\dpt}$ can be
  expressed in terms of the equilavent hyperplane $\dpt^\top\vec{w}
  \ge d$ parameterized by a normal vector $\vec{w}$ and displacement
  $\displacement = \left(\vec{b}-\xtarget\right)^\top\vec{w}$ as
  \begin{equation}
    \begin{cases}
      \displacement \cdot \|\vec{w}\|_{\frac{p}{p-1}}^{-1}  \:, & \mbox{if } \displacement > 0 \\
      0 \:, & \mbox{otherwise}
    \end{cases}
    \label{eq:hyperPMinimizer}
  \end{equation}
  for all $1 < p < \infty$ and is
  \begin{equation}
    \begin{cases}
      \displacement \cdot \|\vec{w}\|_{1}^{-1}  \:, & \mbox{if } \displacement > 0 \\
      0 \:, & \mbox{otherwise}
    \end{cases}
    \label{eq:hyperInfMinimizer}
  \end{equation}
  for $p = \infty$.
\end{lemma}
\begin{proof}
  For $1 < p < \infty$, minimizing \lpCost{p} on the
  halfspace \halfspace{\vec{w}}{\vec{b}} is equivalent to finding a
  minimizer for
  \[
    \min_{\dpt}{\frac{1}{p}\sum_{i=1}^{\dims}{\left|\dptc_i\right|^p}} \quad \textrm{s.t.} \quad \dpt^\top\vec{w} \le \displacement \enspace.
  \]
  Clearly, if $\displacement \le 0$ then the vector $\vec{0}$ (corresponding to
  \xtarget\ in the transformed space) trivially satisfies the
  constraint and minimizes the cost function with cost $0$ which
  yields the second case of Eq.~\eqref{eq:hyperPMinimizer}. For the case
  $\displacement > 0$, we construct the
  Lagrangian
  \[
  \lagrangian{\dpt,\lambda} \defAs
  \frac{1}{p}\sum_{i=1}^{\dims}{\left|\dptc_i\right|^p} - \lambda
  \left( \dpt^\top\vec{w} - \displacement \right) \enspace.
  \]
  Differentiating this with respect to $\dpt$ and setting that partial
  derivative equal to zero yields
  \[
    \dptc^\ast_i = \sign(w_i)\left(\lambda|w_i|\right)^{\frac{1}{p-1}} \enspace.
  \]
  Plugging this back into the Lagrangian yields
  \[
    \lagrangian{\dpt^\ast,\lambda} = \frac{1-p}{p}\lambda^{\frac{p}{p-1}}\sum_{i=1}^{\dims}{\left|w_i\right|^{\frac{p}{p-1}}} + \lambda \displacement \enspace,
  \]
  which we now differentiate with respect to $\lambda$ and set the
  derivative equal to zero to yield
  \[
    \lambda^\ast = \left(\frac{\displacement}{\sum_{i=1}^{\dims}{\left|w_i\right|^{\frac{p}{p-1}}}}\right)^{p-1} \enspace.
  \]
  Plugging this solution into the formula for $\dpt^\ast$ yields the solution
  \[
    \dptc^\ast_i = \sign(w_i)  \left(\frac{\displacement}{\sum_{i=1}^{\dims}{\left|w_i\right|^{\frac{p}{p-1}}}}\right)  |w_i|^{\frac{1}{p-1}} \enspace.
  \]
  The \LP[p] cost of this optimal solution is given by
  \[
    \lpCostFunc{p}{\dpt^\ast} = \displacement \cdot \|\vec{w}\|_{\frac{p}{p-1}}^{-1}
    \enspace,
  \]
  which is the first case of Eq.~\eqref{eq:hyperPMinimizer}.

  For $p = \infty$, once again if $\displacement \le 0$ then the
  vector $\vec{0}$ trivially satisfies the constraint and minimizes
  the cost function with cost $0$ which yields the second case of
  Eq.~\eqref{eq:hyperInfMinimizer}. For the case $\displacement > 0$, we
  use the geometry of hypercubes (the equi-cost balls of a \LP[\infty]
  cost function) to derive the second case of
  Eq.~\eqref{eq:hyperInfMinimizer}. For any optimal solution must occur
  at a point where the hyperplane given by $\dpt^\top\vec{w} =
  \vec{b}^\top\vec{w}$ is tangent to a hypercube about \xtarget---this
  can either occur along a side (face) of the hypercube or at a
  corner. However, if the plane is tangent along a side (face) it is
  also tangent at a corner of the hypercube. Hence, there is always an
  optimal solution at some corner of optimal cost hypercube.

  At a corner of the hypercube, we have the following property:
  \[
    |\dptc_1^\ast| = |\dptc_2^\ast| = \ldots = |\dptc_\dims^\ast| \enspace;
  \]
  that is, the magnitude of all coordiates of this optimal solution is
  the same value. Further, the sign of the optimal solution's
  \nth{i}{th} coordinate must agree with the sign of the hyperplane's
  \nth{i}{th} coordinate, $\vec{w}_i$. These constraints, along with the hyperplane constraint, lead to the following formula for an optimal solution:
  \[
    \dptc_i = \displacement \cdot \sign(w_i) \|\vec{w}\|_1^{-1}
    \enspace .
  \]
  The \LP[\infty] cost of these solutions is simply
  \[
    \displacement \cdot \|\vec{w}\|_1^{-1} \enspace.
  \]
\end{proof}

For the proof of Theorem~\ref{thm:LP-exponential}, we use the
orthants (centered at \xtarget)---an \term{orthant} is the
$\dims$-dimensional generalization of a quadrant in $2$-dimensions.
There are $2^\dims$ orthants in a $\dims$-dimensional space. We
represent each orthant by it's \term{canonical representation} which
is a vector of $\dims$ positive or negative ones; \ie, the orthant
represented by $\vec{a} = (\pm 1, \pm 1, \ldots, \pm 1)$ contains the
point $\xtarget + \vec{a}$ and is the set
of all points $\dpt$ satisfying:
\[
  \dpt_i \in \begin{cases}
      [0,+\infty]\:, & \mbox{if } \vec{a} = +1 \\
      [-\infty,0]\:, & \mbox{if } \vec{a} = -1
    \end{cases}\enspace.
\]

\newcommand{\orthant}[1]{\function{orth}{#1}}

\begin{proofof}{Theorem~\ref{thm:LP-exponential}}\mbox{}
  Suppose a query-based algorithm submits $N$ membership queries
  $\dpt^{1},\ldots,\dpt^{N}\in\reals^{\dims}$ to the classifier.
  Again, for the algorithm to be \multGoal-optimal, these queries must
  constrain all consistent classifiers
  $\hat{\classSpace}^\mathrm{convex,\posLbl}$ to have a common point
  among their \kIMAC[\multGoal] sets. The responses described above
  are consistent with the classifier $\classifier$ defined as
  \begin{equation}
    \prediction{\dpt} = \begin{cases}
      +1\:, & \mbox{if } \lpCostFunc{p}{\dpt} < \maxCost \\
      -1\:, & \mbox{otherwise}
    \end{cases}\enspace;
    \label{eq:LinfPredictor}
  \end{equation}
  For this classifier, $\xplus$ is convex since $\lpCost{p}$ is a
  convex function for $p \ge 1$, $\ball{\minCost}{\lpCost{p}} \subset
  \xplus$ since $\minCost < \maxCost$, and
  $\ball{\maxCost}{\lpCost{p}} \not\subset \xplus$ since $\xplus$ is
  the open \maxCost-ball whereas $\ball{\maxCost}{\lpCost{p}}$ is the
  closed \maxCost-ball.  Moreover, since $\xplus$ is the open
  \maxCost-ball, $\nexists \; \dpt \in \xminus \; \mathrm{s.t.} \;
  \lpCostFunc{p}{\dpt} < \maxCost$ therefore
  $\function{\MAC}{\classifier,\lpCost{p}}=\maxCost$, and any
  \multGoal-optimal points $\dpt^\prime \in
  \function{\multIMAC}{\classifier,\lpCost{p}}$ must satisfy $\maxCost
  \le \lpCostFunc{p}{\dpt^\prime} \le (1 + \multGoal) \maxCost$.

  Now consider an alternative classifier $\funcName{g}$ that responds
  identically to $\classifier$ for $\dpt^{1},\ldots,\dpt^{N}$ but has
  a different convex positive set $\xplus[\funcName{g}]$. Without loss
  of generality suppose the first $M \le N$ queries are positive and
  the remaining are negative. Here we consider a set which is a convex
  hull of the orthants of all $M$ positive queries; that is,
  \[
  \aset{G}
  = \function{conv}{\orthant{\dpt^1}\cap\xplus,
    \orthant{\dpt^2}\cap\xplus, \ldots, \orthant{\dpt^M}\cap\xplus}
  \]
  where $\orthant{\dpt}$ is some orthant that $\dpt$ lies with in
  relative to \xtarget\ (a data point may lie within more than one
  orthant but we need only select any orthant that contains it in
  order to cover it). By intersecting each data point's orthant with
  the set \xplus\ and taking the convex hull of these regions,
  \aset{G} is convex , contains \xtarget\, and is a subset of \xplus\
  that is also consistent with all the query responses of \classifier;
  \ie, each of the $M$ positive queries are in \xplus[\funcName{g}]
  and all the negative queries are in \xminus[\funcName{g}]. Moreover,
  \aset{G} is a superset of the convex hull of the $M$ positive
  queries.
  Thus, by finding the largest enclosed \LP[p] ball within the
  \aset{G}, we upper bound $\function{\MAC}{\funcName{g},\lpCost{p}}$.

  We now represent each orthant as a vertex in a $D$-dimensional
  hypercube graph---the Hamming distance between any pair of orthants
  is the number of different coordinates in their canonical
  representations and two orthants are adjacent in the graph if and
  only if they have Hamming distance of 1. Using this notion of
  Hamming distance, we will seek a $K$-covering of the hypercube. We
  refer to the orthants used in \aset{G} to cover the $M$ positive
  queries as \term{covering orthants} and their corresponding vertices
  form a covering of the hypercube. Suppose the $M$ covering orthants
  are sufficient for a $K$ covering but not $K-1$ covering; then there
  must be at least one vertex not in the covering that has at least a
  $K$ Hamming distance to every vertex in the covering. This vertex
  corresponds to an empty orthant that differs from all covered
  orthants in at least $K$ coordinates of their canonical vertices.
  Without loss of generality, suppose this uncovered orthant has the
  canonical vertex of all postitive ones which we scale to $\maxCost
  (+1,+1,\ldots,+1)$. Consider the hyperplane with normal vector
  $\vec{w} = (+1,+1,\ldots,+1)$ and displacement
  \begin{eqnarray*}
  \displacement &=&
  \begin{cases} \maxCost (\dims - K)^{\frac{p-1}{p}}\: & \mbox{if } 1 < p < \infty
    \\ \maxCost (\dims - K)\: & \mbox{if } p = \infty \end{cases}
  \end{eqnarray*}
  that specifies
  the function $\function{s}{\dpt} = \dpt^\top\vec{w}-\displacement =
  \sum_{i=1}^{\dims}{\dpt_i} - \displacement$.  For this hyperplane, the vertex
  $\maxCost (+1,+1,\ldots,+1)$ yields
  \[
    \function{s}{\maxCost(+1,+1,\ldots,+1)} = \maxCost\dims - \displacement > 0\; .
  \]
  Also for any orthant $\vec{a}$ with Hamming distance at least $K$
  from this uncovered orthant, we have that for any $\dpt \in
  \orthant{\vec{a}} \cap \xplus$, by definition of the orthant and
  \xplus, the function \funcName{s} yields
  \begin{eqnarray*}
    \function{s}{\dpt} & = & \sum_{i=1}^{\dims}{\dptc_i} - \displacement \\
    & = & \sum_{\set[\vec{a_i}=+1]{i}}{\underbrace{\dptc_i}_{\ge 0}} + \sum_{\set[\vec{a_i}=-1]{i}}{\underbrace{\dptc_i}_{\le 0}} - \displacement \enspace.
  \end{eqnarray*}
  Since all the terms in the second summation are non-postive, the
  second sum is at most 0. Further, by maximizing the first summation,
  we upper bound $\function{s}{\dpt}$. The summation
  $\sum_{\set[\vec{a_i}=+1]{i}}{\dptc_i}$ (with the constraint that
  $\|\dpt\|_p < \maxCost$) has at most $D-K$ terms and is maximized by
  $\dptc_i = \maxCost(\dims-K)^{-1/p}$ (or $\dptc_i = \maxCost$ for
  $p=\infty$) for which the first summation is upper bounded by
  $\maxCost(\dims-K)^{\frac{p-1}{p}}$ or $\maxCost(\dims-K)$ for
  $p=\infty$; \ie\ it is upper bounded by $\displacement$. Thus we see
  that
  \[
  \function{s}{\dpt} \le 0 \enspace.
  \]
  Thus, this hyperplane seperates the scaled vertex
  $\maxCost(+1,+1,\ldots,+1)$ from each set $\orthant{\vec{a}} \cap
  \xplus$ where $\vec{a}$ is the canonical representation of any
  orthant with a Hamming distance of at least $K$. Thus, this
  hyperplane also seperates the scaled vertex from $\aset{G}$ by the
  properties of the convex hull.  Since the displacement $\maxCost
  (\dims - K) > 0$, by applying
  Lemma~\ref{thm:hyperplane-lp-minimizer}, this separating hyperplane
  upper bounds the cost of the largest \LP[p] ball enclosed in
  \aset{G} as
  \[
  \function{\MAC}{\funcName{g},\lpCost{p}} \le \maxCost (\dims - K)^{\frac{p-1}{p}} \cdot
  \|\vec{w}\|_{\frac{p}{p-1}}^{-1} =
  \maxCost \left(\frac{\dims - K}{\dims}\right)^{\frac{p-1}{p}}
  \]
  for $1 < p < \infty$ and
  \[
  \function{\MAC}{\funcName{g},\lpCost{p}} \le \maxCost (\dims - K) \cdot \|\vec{1}\|_{1}^{-1} = \maxCost \frac{\dims - K}{\dims}
  \]
  for $p = \infty$. Since we have an upper bound on the \MAC\ of \funcName{g} and the \MAC\ of \classifier\ is \maxCost, in order to have a common \kIMAC\ between these classifiers, we must have
  \[
    (1+\multGoal) \ge \begin{cases}
      \left(\frac{\dims}{\dims-K}\right)^{\frac{p-1}{p}}\:, & \mbox{if } 1 < p < \infty \\
      \frac{\dims}{\dims-K}\:, & \mbox{if } p = \infty
    \end{cases} \enspace.
  \]
  Solving for the value of $K$ required to achieve a desired accuracy
  of $1+\multGoal$ we have
  \[
    K \le \begin{cases}
      \frac{(1+\multGoal)^{\frac{p}{p-1}}-1}{(1+\multGoal)^{\frac{p}{p-1}}} \dims\:, & \mbox{if } 1 < p < \infty \\
      \frac{\multGoal}{1+\multGoal} \dims\:, & \mbox{if } p = \infty
    \end{cases} \enspace,
  \] 
  which bounds the size of the covering required to achieve the
  desired accuracy.

  For the case $1 < p < \infty$, by Lemma~\ref{lma:hypercubeCover},
  there must be
  \[
  M \ge \exp\left\{ \ln(2) \cdot \dims\left(1 -
      \entropy{\frac{(1+\multGoal)^{\frac{p}{p-1}}-1}{(1+\multGoal)^{\frac{p}{p-1}}}}\right) \right\}
  \]
  vertices of the hypercube in the covering to achieve any desired
  accuracy $0 < \multGoal < 2^\frac{p-1}{p} - 1$, for which
\begin{eqnarray*}  
\frac{(1+\multGoal)^{\frac{p}{p-1}}-1}{(1+\multGoal)^{\frac{p}{p-1}}}
  &<& \frac{1}{2}
\end{eqnarray*}
  as required by the Lemma. Moreover, since $0 <
  \entropy{\delta} < 1$ for any $0 < \delta < 1$,
\begin{eqnarray*}
  \alpha_{p,\multGoal} &=& \exp\left\{\ln(2) \left(1 -
    \entropy{\frac{(1+\multGoal)^{\frac{p}{p-1}}-1}{(1+\multGoal)^{\frac{p}{p-1}}}}\right) \right\}
  > 1
\end{eqnarray*} and we have
  \[
  M > \alpha_{p,\multGoal}^\dims \enspace.
  \]

  Similarly for $p=\infty$, Lemma~\ref{lma:hypercubeCover} can be applied yielding
  \[
    M \ge 2^{\dims\left(1 - \entropy{\frac{\multGoal}{1+\multGoal}}\right)}
  \]
  to achieve any desired accuracy $0 < \multGoal < 1$ (for which
  $\multGoal/(1+\multGoal) < 1/2$ as required by the
  Lemma). Again, by the properties of entropy the constant
  $\alpha_{\infty,\multGoal} = 2^{\left(1 -
    \entropy{\frac{\multGoal}{1+\multGoal}}\right)} > 1$ for $0 < \multGoal
  < 1$ and we have
  \[
  M > \alpha_{\infty,\multGoal}^\dims \enspace.
  \]
\end{proofof}

\section{Proof of Theorem~\ref{thm:L2-exponential}}
\label{sec:L2-exponential}

\newcommand{\sphere}[1]{\aset{S}^{#1}}
\newcommand{\scap}[2]{\aset{C}^{#1}_{#2}}
\newcommand{\gammaFunc}[1]{\Gamma\left(#1\right)}
\newcommand{\vol}[1]{\function{vol}{#1}}
\newcommand{\surfArea}[1]{\function{surf}{#1}}

\begin{figure}[t]
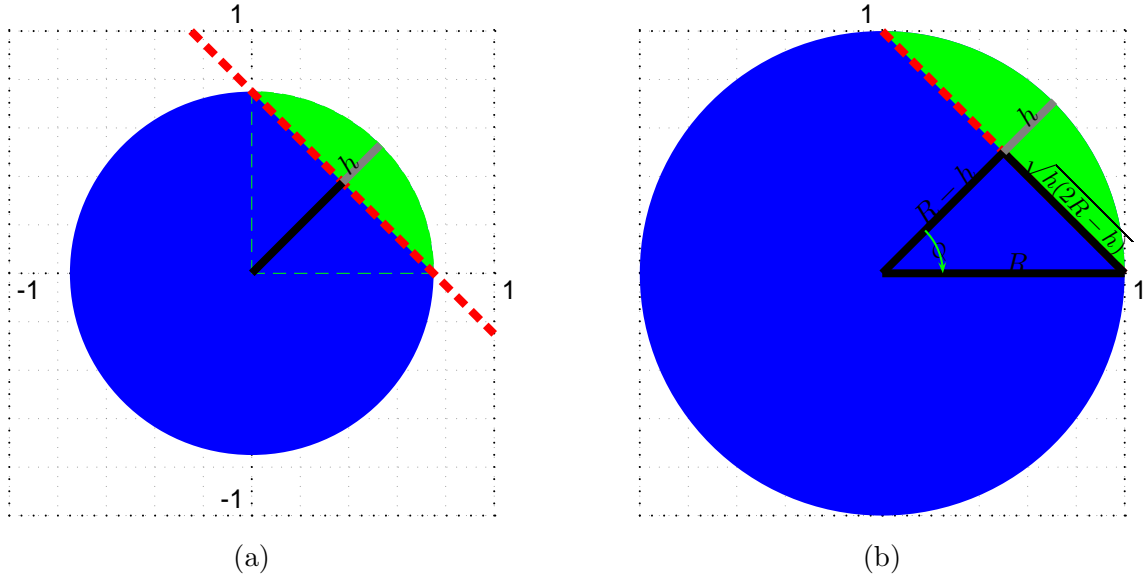

\begin{minipage}{.45\linewidth}
\begin{center}
\renewcommand{\picwidth}{0.47\linewidth}
\pspicture(-\picwidth,-\picwidth)(\picwidth,\picwidth)
  \psset{unit=\picwidth}
  \psgrid[unit=\picwidth,griddots=20,subgriddots=5](0,0)(-1,-1)(1,1)
  \pscircle*[linecolor=blue](0,0){0.75}
  \psarc*[showpoints=true,linecolor=green](0,0){0.75}{0}{90}
  \psline[linecolor=red,linestyle=dashed,linewidth=3pt]{}(-0.25,1.0)(1.0,-0.25)
  \psline[linecolor=black,linewidth=3pt](0,0)(0.375,0.375)
  \psline[linecolor=gray,linewidth=3pt](0.375,0.375)(0.5303,0.5303)
  \rput[br]{*45}(0.45,0.45){$h$}
\endpspicture\\[0.5em]
(a)
\end{center}
\end{minipage} \hfill
\begin{minipage}{.45\linewidth}
\begin{center}
\renewcommand{\picwidth}{0.47\linewidth}
\pspicture(-\picwidth,-\picwidth)(\picwidth,\picwidth)
  \psset{unit=\picwidth}
  \psgrid[unit=\picwidth,griddots=20,subgriddots=5](0,0)(-1,-1)(1,1)
  \pscircle*[linecolor=blue](0,0){1.0}
  \psarc*[linecolor=green](0,0){1.0}{0}{90}
  \psline[linecolor=red,linestyle=dashed,linewidth=3pt]{}(0,1.0)(1.0,0)
  \psline[linecolor=black,linewidth=3pt](0,0)(0.5,0.5)
  \psline[linecolor=black,linewidth=3pt](0,0)(1,0)
  \psline[linecolor=black,linewidth=3pt](0.5,0.5)(1,0)
  \psline[linecolor=gray,linewidth=3pt](0.5,0.5)(0.7101,0.7101)
  \rput[br]{*45}(0.65,0.65){$h$}
  \rput[br]{*45}(0.4,0.4){$R-h$}
  \rput[br]{*0}(0.6,0){$R$}
  \rput[br]{*315}(0.95,0.05){\footnotesize $\sqrt{h(2R-h)}$}
  \rput{*45}(0.2309699,0.0956709){$\phi$}
  \psarcn[linecolor=green]{->}{.25}{45}{0}
\endpspicture\\[0.5em]
(b)
\end{center}
\end{minipage}
\caption{This figure depictions the geometry of spherical caps.
  (a) A spherical cap of height $h$ is shown that is created by a plane
  passing through the sphere. The green region represents the area of the cap.
  (b) We see the geometry of the spherical cap. Notice that the
  intersecting hyperplane forms a right triangle with the centroid of
  the hypershere. The length of the first side of that triangle is
  $R-h$, it's hypotenuse is length $R$, and its other side is length
  $\sqrt{h(2R-h)}$. The half angle $\phi$ of the right circular cone
  can also be used to parameterize the cap.}
\label{fig:sphericalcaps}
\end{figure}

For this proof, we build on previous results for covering
hyperspheres. The proof is based on the following covering number
result by Wyner and Shannon which bounds the minimum number of
spherical caps required to cover a hypersphere. A $\dims$-dimensional
\term{spherical cap} is the region formed by the intersection of a
halfspace and a hypersphere facing away from the center of the
hypersphere as depicted in Figure~\ref{fig:sphericalcaps}. This cap is
parameterized by the hypersphere's radius $R$ and the half-angle
$\phi$ about a central radius (through the peak of the cap) as in the
right-most diagram of Figure~\ref{fig:sphericalcaps}.

Based on these formula, we now derive a bound on the number of
spherical caps of half-angle $\phi$ required to cover the sphere,
mirroring the result due to \citet{Wyner65}.

\begin{lemma}
  \label{lma:capCovering}
  \textbf{(Result based on Wyner 1965)} Covering the surface of
  $\dims$-dimensional hypersphere of radius $R$ requires at least
  \[
    \left(\frac{1}{\sin \phi} \right)^{\dims-2}
  \]
  spherical caps of half-angle $\phi$.
\end{lemma}
\begin{proof}
  In \emph{Capabilities of Bounded Discrepancy Decoding}, Wyner showed
  that the minimal number, $M$, of spherical caps of half-angle $\phi$
  required to cover $\dims$-dimensional hypersphere of radius $R$ is
  given by
  \begin{equation*}
    M \ge \frac{\dims \sqrt{\pi} \gammaFunc{\frac{\dims+1}{2}}}{(\dims - 1) \gammaFunc{1+\frac{\dims}{2}}} \left[\int_{0}^{\phi}{\sin^{\dims-2}(t)dt} \right]^{-1} \enspace.
  \end{equation*}
  This result follows directly from computing the surface area of the
  hypersphere and the spherical caps.
 
  We continue by lower bounding the above integral for a looser but
  more interpretable bound. Integrals of the form
  $\int_0^\phi{\sin^\dims(t) dt}$ also arise in computing the volume
  of a spherical cap. This volume (and thus the integral) can be
  bounded by enclosing the cap within a
  hypersphere; \cf\ \citet{Ball-FoG-1997}. This yields the following bound:
  \begin{equation*}
    \int_{0}^{\phi}{\sin^\dims(t)dt} \le \frac{\sqrt{\pi} \gammaFunc{\frac{\dims+1}{2}}}{\gammaFunc{1+\frac{\dims}{2}}} \cdot \sin^{\dims} \phi \enspace.
  \end{equation*}
  Using this bound on the integral, our bound on the size of the
  covering is
  \[
    M \ge \frac{\dims \sqrt{\pi} \gammaFunc{\frac{\dims+1}{2}}}{(\dims - 1) \gammaFunc{1+\frac{\dims}{2}}} \left[\frac{\sqrt{\pi} \gammaFunc{\frac{\dims-1}{2}}}{\gammaFunc{\frac{\dims}{2}}} \cdot \sin^{\dims-2} \phi \right]^{-1} \enspace.
  \]
  Now using properties of the gamma function, it can be shown that
  $\frac{\gammaFunc{\frac{\dims+1}{2}}
    \gammaFunc{\frac{\dims}{2}}}{\gammaFunc{1+\frac{\dims}{2}}
    \gammaFunc{\frac{\dims-1}{2}}} = \frac{\dims-1}{\dims}$ so that
  after canceling terms we arrive at our result: 
  \[
    M \ge \left(\frac{1}{\sin \phi} \right)^{\dims-2} \enspace.
  \]
\end{proof}

\begin{proofof}{Theorem~\ref{thm:L2-exponential}}\mbox{}
  Suppose a query-based algorithm submits $N < \dims + 1$ membership
  queries $\dpt^{1},\ldots,\dpt^{N}\in\reals^{\dims}$ to the
  classifier. For the algorithm to be \multGoal-optimal, these queries
  must constrain all consistent classifiers
  $\hat{\classSpace}^\mathrm{convex,\posLbl}$ to have a common point
  among their \kIMAC[\multGoal] sets. Suppose that all the responses
  are consistent with the classifier $\classifier$ defined as
  \begin{equation}
    \prediction{\dpt} = \begin{cases}
      +1\:, & \mbox{if } \lpCostFunc{2}{\dpt} < \maxCost \\
      -1\:, & \mbox{otherwise}
    \end{cases}\enspace;
    \label{eq:L2Predictor}
  \end{equation}
  For this classifier, $\xplus$ is convex since $\lpCost{2}$ is a
  convex function, $\ball{\minCost}{\lpCost{2}} \subset \xplus$ since
  $\minCost < \maxCost$, and $\ball{\maxCost}{\lpCost{2}} \not\subset
  \xplus$ since $\xplus$ is the open \maxCost-ball whereas
  $\ball{\maxCost}{\lpCost{2}}$ is the closed \maxCost-ball.
  Moreover, since $\xplus$ is the open \maxCost-ball, $\nexists \;
  \dpt \in \xminus \; \mathrm{s.t.} \; \lpCostFunc{2}{\dpt} <
  \maxCost$ therefore
  $\function{\MAC}{\classifier,\lpCost{2}}=\maxCost$, and any
  \multGoal-optimal points $\dpt^\prime \in
  \function{\multIMAC}{\classifier,\lpCost{2}}$ must satisfy $\maxCost
  \le \lpCostFunc{2}{\dpt^\prime} \le (1 + \multGoal) \maxCost$.
  
  Now consider an alternative classifier $\funcName{g}$ that responds
  identically to $\classifier$ for $\dpt^{1},\ldots,\dpt^{N}$ but has
  a different convex positive set $\xplus[\funcName{g}]$. Without loss
  of generality suppose the first $M \le N$ queries are positive and
  the remaining are negative. Let $\aset{G} =
  \function{conv}{\dpt^1,\ldots,\dpt^M}$; that is, the convex hull of
  the $M$ positive queries.  We will assume $\xtarget \in \aset{G}$
  since if it is not, then we constuct the set $\xplus[\funcName{g}]$
  as in the proof for Theorems~\ref{thm:mlower} and~\ref{thm:alower}
  above and achieve $\function{\MAC}{\classifier,\lpCost{2}}=\minCost$
  thereby showing our desired result. Now consider the points
  $\vec{z}^i = \maxCost \frac{\dpt^i}{\lpCostFunc{2}{\dpt^i}}$; \ie,
  the projection of each of the positive queries onto the surface of
  the \LP[2] ball $\ball{\maxCost}{\lpCost{2}}$. Since each positive
  query lies along the line between $\xtarget$ and its projection
  $\vec{z}^i$, by convexity and the fact that $\xtarget \in \aset{G}$,
  we have $\aset{G} \subset
  \function{conv}{\vec{z}^1,\vec{z}^2,\ldots,\vec{z}^M}$. We will call
  this enlarged hull $\hat{\aset{G}}$. These $M$ projected points
  $\set{\vec{z}^i}$ must form a covering of the $\maxCost$-hypersphere
  as the locii of caps of half-angle $\phi^\ast = \arccos\left(
    (1+\multGoal)^{-1} \right)$. If not, then there exists some
  point on the surface of this hypersphere that is at least an angle
  $\phi^\ast$ from all $\vec{z}^i$ points and the resulting
  $\phi^\ast$-cap centered at this uncovered point is not in
  $\hat{\aset{G}}$ (since a cap is defined as the intersection of the
  hypersphere and a halfspace). Moreover, by definition of the
  $\phi^\ast$-cap, it achieves a minimal \LP[2] cost of $\maxCost \cos
  \phi^\ast$. Thus, if we fail to achieve a $\phi^\ast$-covering of
  the $\maxCost$-hypersphere, the alternative classifier
  $\funcName{g}$ has $\function{\MAC}{\funcName{g},\lpCost{2}} <
  \maxCost \cos \phi^\ast = \maxCost/(1+\multGoal)$ and any
  $\dpt \in \function{\multIMAC}{\funcName{g},\lpCost{2}}$ must have
  \[
  \lpCostFunc{2}{\dpt} \le (1+\multGoal)\MAC < (1+\multGoal)
  \frac{\maxCost}{1+\epsilon} = \maxCost \enspace,
  \]
  whereas any $\vec{y} \in \function{\multIMAC}{\classifier,\adCost}$
  must have $\adCostFunc{\vec{y}} \ge \maxCost$.  Thus, we would have
  $\function{\multIMAC}{\classifier,\adCost} \cap
  \function{\multIMAC}{\funcName{g},\adCost} = \emptyset$ and thus
  fail to achieve \multGoal-multiplicative optimality. Thus, we have
  shown that an $\phi^\ast$-covering is necessary for
  \multGoal-multiplicative optimality. However, from
  Lemma~\ref{lma:capCovering}, to have a $\phi^\ast$-covering we must
  have
  \[
  M \ge \left(\frac{1}{\sin \phi^\ast} \right)^{\dims-2} \enspace .
  \]
  Using the trigonometric identity $\sin\left(\arccos(x)\right) =
  \sqrt{1-x^2}$ we can substitute for $\phi^\ast$ and find
  \begin{eqnarray*}
    M & \ge & \left(\frac{1}{\sin\left( \arccos\left( \frac{1}{1+\multGoal} \right) \right)} \right)^{\dims-2} \\
    & \ge & \left(\frac{(1+\multGoal)^2}{(1+\multGoal)^2-1} \right)^{\frac{\dims-2}{2}} \enspace.
  \end{eqnarray*}
\end{proofof}

\end{document}